%% file: main.tex
\documentclass[letterpaper]{article}

\usepackage[utf8]{inputenc} 
\usepackage[T1]{fontenc}    
\usepackage{hyperref}       
\usepackage{url}            
\usepackage{booktabs}       
\usepackage{amsfonts}       
\usepackage{nicefrac}       
\usepackage{microtype}      
\usepackage{xcolor}         
\usepackage{times}
\usepackage{helvet}
\usepackage{courier}
\usepackage{graphicx}
\usepackage[margin=1.25in]{geometry}
\usepackage{natbib}
\usepackage{enumerate}
\usepackage{caption}
\usepackage{soul}
\usepackage[utf8]{inputenc}
\usepackage{amsmath}
\usepackage{amssymb}
\usepackage{amsthm}
\newtheorem{theorem}{Theorem}
\newtheorem{lemma}{Lemma}
\newtheorem{claim}{Claim}
\newtheorem{definition}{Definition}

\usepackage{algorithm}
\usepackage{float}
\usepackage[noend]{algpseudocode}
\usepackage{enumitem}
\usepackage{nicefrac}
\usepackage{xspace}
\usepackage{subcaption}
\usepackage{thm-restate}

\newtheorem{corollary}{Corollary}[theorem]

\allowdisplaybreaks

\newcommand{\vect}{\textbf}
\newcommand{\greedyalg}{\texttt{greedy-fair-bi}\xspace}
\newcommand{\contisublong}
{\texttt{decreasing-threshold-procedure}\xspace}
\newcommand{\contisub}
{\texttt{DTP}\xspace}
\newcommand{\contialglong}
{\texttt{cont-thresh-greedy-bi}\xspace}
\newcommand{\convc}{\texttt{convert-continuous}\xspace}
\newcommand{\contialg}
{\texttt{cont-bi}\xspace}
\newcommand{\threalglong}
{\texttt{threshold-fairness-bi}\xspace}
\newcommand{\conv}
{\texttt{convert-fair}\xspace}
\newcommand{\sumN}{\sum_{c\in[N]}}

\newcommand{\mE}{\mathbb{E}}

\newif\ifnotes\notestrue 
\ifnotes
\newcommand{\samson}[1]{\textcolor{purple}{{\bf (Samson:} {#1}{\bf ) }} \marginpar{\tiny\bf
             \begin{minipage}[t]{0.5in}
               \raggedright S:
            \end{minipage}}}            							
\else
\newcommand{\samson}[1]{}
\fi

\usepackage{authblk}
\title{Fair Submodular Cover}
\author[1]{Wenjing Chen}
\author[1]{Shuo Xing}
\author[1]{Samson Zhou}
\author[1]{Victoria G. Crawford}
\affil[1]{Department of Computer Science \& Engineering, Texas A\&M University}
\date{}                     
\begin{document}

\maketitle

\begin{abstract}
Submodular optimization is a fundamental problem with many applications in machine learning, often involving decision-making over datasets with sensitive attributes such as gender or age. In such settings, it is often desirable to produce a diverse solution set that is fairly distributed with respect to these attributes. Motivated by this, we initiate the study of Fair Submodular Cover (FSC), where given a ground set $U$, a monotone submodular function $f:2^U\to\mathbb{R}_{\ge 0}$, a threshold $\tau$, the goal is to find a balanced subset of $S$ with minimum cardinality such that $f(S)\ge\tau$. We first introduce discrete algorithms for FSC that achieve a bicriteria approximation ratio of $(\frac{1}{\varepsilon}, 1-O(\varepsilon))$. We then present a continuous algorithm that achieves a $(\ln\frac{1}{\varepsilon}, 1-O(\varepsilon))$-bicriteria approximation ratio, which matches the best approximation guarantee of submodular cover without a fairness constraint. Finally, we complement our theoretical results with a number of empirical evaluations that demonstrate the effectiveness of our algorithms on instances of maximum coverage.

\end{abstract}

\section{Introduction}
From high-volume applications such as online advertising and smart devices to high-impact applications such as credit assessment, medical diagnosis, and self-driving vehicles, machine learning algorithms are increasingly prevalent in technologies and decision-making processes in the modern world. 
However, the amount of automated decision-making has resulted in concerns about the risk of unintentional bias or discrimination~\citep{chouldechova2017fair,kleinberg2018human,berk2021fairness}. 
For example, \citep{ChierichettiKLV17} noted that although machine learning algorithms may not be biased or unfair by design, they may nevertheless acquire and amplify biases already present in the training data available to the algorithms.
Consequently, there has recently been significant focus on achieving algorithmic fairness for a number of fundamental problems, such as classification~\citep{ZafarVGG17}, clustering~\citep{ChierichettiKLV17}, data summarization~\citep{CelisKSDKV18}, and matchings~\citep{ChierichettiKLV19}. 

In this work, we focus on fairness within submodular optimization. 
Submodular functions informally satisfy a diminishing returns property that is exhibited by many objective functions for fundamental optimization problems in machine learning. 
Thus, submodular optimization naturally arises in a wide range of applications, such as clustering and facility location~\citep{GomesK10,LindgrenWD16}, document summarization~\citep{LinB11,WeiLKB13,SiposSSJ12}, image processing~\citep{IyerB19}, principal component analysis~\citep{KhannaGPK15}, and recommendation systems~\citep{LeskovecKGFVG07,ElAriniG11,BogunovicMSC17,MitrovicBNTC17,YuXC18,AvdiukhinMYZ19,YaroslavtsevZA20}. In particular, a function $f:2^U\to\mathbb{R}$ is submodular if for every $X\subset Y\subset U$ and for every $x\in U\setminus Y$, we have $f(X\cup\{x\})-f(X)\ge f(Y\cup\{x\})-f(Y)$. We further assume $f$ is monotone, i.e., $f(Y)\ge f(X)$ for every $X\subset Y$. 

Though various definitions have been proposed, there is no universal notion of fairness; indeed, determining the correct notion of fairness is an ongoing active line of research. 
In fact, \cite{KleinbergMR17} showed that three common desiderata of fairness (probabilistic calibration across classes, numerical balance across classes, and statistical parity) are inherently incompatible.  
Nevertheless, there has been significant focus recently~\cite{ChierichettiKLV17,CelisHV18,CelisKSDKV18,CelisSV18,ChierichettiKLV19,el2020fairness,HalabiTNV24} on the fairness notion that demands a solution to be balanced with respect to a sensitive attribute, such as ethnicity or gender.

Fair submodular maximixation (FSM) has been considered under both a cardinality constraint and a matroid constraint \citep{CelisHV18,HalabiMNTT20}. However, to the best of our knowledge, there has been no previous work studying fairness for the submodular cover problem. Given oracle access to a submodular function $f:2^U\rightarrow\mathbb{R}$ and a threshold $\tau$, the goal of submodular cover is to identify a subset $S\subset U$ of minimal size such that $f(S)\ge\tau$. If we additionally assume that each element in $U$ is associated with a color $c$ that denotes a protected attribute, which partitions $U$ into disjoint groups $U_1,\ldots,U_N$. 
We denote the partition as $\mathcal{P}=\{U_1,U_2,\ldots,U_N\}$. 
Given upper and lower bounds $u_c$ and $l_c$ for each color $c$, we say that a solution $S\subset U$ is fair if $l_c\le|S\cap U_c|\le u_c$.  This definition of fairness incorporates multiple other existing notions of fairness, such as diversity rules~\cite{biddle2017adverse,CohoonCRL13}, statistical parity~\cite{DworkHPRZ12}, or proportional representation rules~\cite{monroe1995fully,BrillLS17}. 

\textbf{Fair submodular cover problem (FSC).} 
In this work, we initiate the study of fairness for submodular cover. We define the problem of submodular cover with fairness constraint as follows. Given input threshold $\tau$, and bounds on the proportion of the elements in each group $u_c$ and $l_c$, FSC is to find the solution set of the following optimization problem.
    \begin{align}
        &\min_{S\in U} |S|\nonumber\\
        s.t.\qquad& p_c|S|\leq|S\cap U_c|\leq q_c|S|\qquad \forall c\in[N]\nonumber\\
        &f(S)\geq\tau.
    \end{align}
To guarantee the existence of feasible subsets, we assume that the inputs satisfy $\sum_{c\in[N]}p_c\leq1$ and $\sum_{c\in[N]}q_c\geq 1$\footnote{Notice that this assumption is necessary: If $\sum_{c\in[N]}p_c> 1$, then $\sum_{c\in[N]}p_c|S|>|S|$. However, by the definition of fairness constraint in FSC, we can get $p_c|S|\leq|S\cap U_c|$. It then follows that $\sum_{c\in[N]}p_c|S|\leq \sum_{c\in[N]}|S\cap U_c|=|S|$. Therefore we can get a contradiction and there are no feasible sets. Similarly, if $\sum_{c\in[N]}q_c\geq 1$, we can also prove that no feasible sets satisfy the constraint.}. In this paper, we propose bicriteria approximation algorithms for FSC, as has been done previously for submodular cover ~\cite{iyer2013submodular,chen2024bicriteria}. An $(\alpha,\beta)$-bicriteria approximation algorithm for FSC returns a solution set $X$ that satisfies 
\[|X|\leq \alpha |OPT|,\qquad p_c|X|\leq|X\cap U_c|\leq  q_c|X|,\qquad f(X)\geq \beta\tau,\]
where $OPT$ is an optimal solution to the instance of FSC.
Notice that the solution of a bicriteria algorithm for FSC always satisfies the fairness constraint for cover. However, the constraint on the function value ($f(X)\geq\tau$) might be violated by a factor of $\beta$ and therefore the solution is not necessarily feasible. But if $\beta$ to close to $1$, we can get a solution that is close to being feasible.

\textbf{Our contributions.}
In addition to being the first to propose the novel problem FSC and develop approximation algorithms for it, we make three additional main contributions in this paper:
\begin{itemize}
    \item[(i)] We develop two algorithms that make use of the dual relationship between FSC and fair submodular maximization (defined in Section \ref{sec:problem_setup}) and convert bicriteria approximation algorithms for FSM into bicriteria approximation algorithms for FSC. The first algorithm, \conv, is designed to convert discrete algorithms for FSM into ones for FSC. In particular, \conv takes an $(\gamma,\beta)$-bicriteria approximation algorithms for FSM and converts it into a $((1+\alpha)\beta,\gamma)$-bicriteria approximation algorithm for FSC. Our second converting algorithm, \convc, takes a continuous $(\gamma,\beta)$-bicriteria approximation algorithm and converts it into a $((1+\alpha)\beta,\frac{(1-\frac{\varepsilon}{2})\gamma-\frac{\varepsilon}{3}}{1+\frac{\varepsilon}{2}+\frac{\varepsilon}{3\gamma}})$-bicriteria approximation algorithm for FSC.
    \item[(ii)] We propose three bicriteria algorithms for FSM that can be paired with our converting algorithms in order to find approximate solutions for FSC that are arbitrarily close to meeting the constraint $f(S)\geq\tau$ in FSC. The first two algorithms are the discrete algorithms \texttt{greedy-fairness-bi} and \threalglong, which both achieve bicriteria approximation ratios of $(1-O(\varepsilon),\frac{1}{\varepsilon})$, but the latter makes less queries to $f$ compared to the former. The third algorithm is a continuous one, \contialglong, which achieves an improved $(1-O(\varepsilon),\ln\frac{1}{\varepsilon}+1)$ bicriteria approximation ratio but requires more queries to $f$.
    \item[(iii)] We perform an experimental comparison between our discrete algorithms for FSC and the standard greedy algorithm (which does not necessarily find a fair solution) on instances of fair maximum coverage in a graph and fair image summarization. We find that our algorithms find fair solutions while the standard greedy algorithm does not, but at a cost of returning solutions of higher cardinality.
\end{itemize}

\input{sections/related.tex}
\input{sections/problem_setup}

\input{sections/convert_thm}

\input{sections/bicriteria_alg}
\input{sections/continuousTG}

\input{sections/experiments}
\section*{Acknowledgements}
Samson Zhou is supported in part by NSF CCF-2335411. 
Samson Zhou would like to thank the Simons Institute for the Theory of Computing, where part of this work was done. 

\def\shortbib{0}
\bibliography{main}

\input{sections/appendix}

\end{document}

%% file: sections/related.tex
\subsection{Related Work}
\cite{CelisHV18} first gave a $(1-1/e)$-approximation algorithm for fair monotone submodular maximization under a cardinality constraint, which is tight given a known $(1-1/e)$ hardness of approximation even without fairness constraints~\cite{NemhauserW78}. This is accomplished by converting their instance of FSM into monotone submodular maximization with a specific type of matroid constraint called a fairness matroid, which we describe in more detail in Section \ref{sec:problem_setup}, and then using existing algorithms for submodular maximization with a matroid constraint. The standard greedy algorithm is a $1/2$ approximation for the submodular maximization with a matroid constraint \citep{fisher1978analysis}, and in addition there exists approximation algorithms using the multilinear extension that achieve a $1-1/e$ approximation guarantee \cite{calinescu2007maximizing,badanidiyuru2014fast}.
\cite{HalabiTNV24} gave a $(1-1/e)$-approximation algorithm for fair monotone submodular maximization under general matroid constraints, though their algorithm only achieves the fairness constraints in expectation. 
Fair submodular optimization has also been under both cardinality and matroid constraints in the streaming setting~\cite{el2020fairness,HalabiTNV24}.

For the classical submodular cover problem without fairness constraints and integral valued $f$, the standard greedy algorithm, where the element of maximum marginal gain is selected one-by-one until $f$ has reached $\tau$, has been shown to have an approximation ratio of $O(\log\max_{e\in U}f(e))$ ~\cite{Wolsey82}. To deal with real-valued $f$ (as in our case), a slight variant of the greedy where we stop at $(1-\varepsilon)\tau$ instead of $\tau$ has been shown to be a $(\ln(1/\varepsilon), 1-\varepsilon)$-bicriteria approximation algorithm \citep{krause2008robust}. 
More generally, note that the submodular cover problem generalizes the classical set cover problem by selecting the submodular function $f$ to be the additive function $f(X)=|X|$. 
Since set cover cannot be approximated to within a $O(\log n)$ factor in polynomial time unless P=NP~\cite{LundY94,RazS97,Feige98,AlonMS06,DinurS14}, the same hardness of approximation applies to submodular cover. 

 Submodular maximization has received relatively more attention than submodular cover. Because of their dual relationship, one approach to developing algorithms for submodular cover is to convert existing ones for submodular maximization into ones for cover \cite{IyerB13,chen2024bicriteria}. In particular, \citeauthor{IyerB13} showed that given a $(\gamma,\beta)$-bicriteria approximation algorithm for submodular maximization with a cardinality constraint, one can produce a $((1+\alpha)\beta,\gamma)$-bicriteria approximation algorithm for submodular cover by making $\log_{1+\alpha}(n)$ guesses for $|OPT|$ in the instance of submodular cover, running the submodular maximization algorithm with the cardinality constraint set to each guess, and returning the smallest solution with $f$ value above $\gamma\tau$. However, this approach does not take into account the fairness constraints and cannot be used to convert algorithms for FSM into ones for FSC.

%% file: sections/problem_setup.tex
\subsection{Preliminaries}
\label{sec:problem_setup}
In this section, we present preliminary definitions and notation that will be used throughout the paper. First, we define the related problem of fair submodular maximization (FSM) of a monotone submodular function $f$ \citep{el2020fairness}, as the following search problem 
    \begin{align*}
        &\max_{S\subseteq U} f(S)\nonumber\\
        s.t.\qquad& l_c\leq|S\cap U_c|\leq u_c\qquad \forall c\in[N]\nonumber\\
        &|S|\leq k,
    \end{align*}
   where $l_c$ and $u_c$ are the bound of cardinality within each small group. 
 Without loss of generality, in this problem, it is assumed that $\sum_{c\in[N]}u_c\geq k$. This is because if $\sum_{c\in[N]}u_c< k$, then $|S|=\sum_{c\in[N]}|S\cap U_c|\leq\sum_{c\in[N]}u_c\leq k$. Therefore, the problem is equivalent to setting $k=\sum_{c\in[N]}u_c$. Since for the cover problem, the objective is to minimize the cardinality of the solution set which means $|S|$ is not fixed as it is in FSM, therefore we introduced the definition of fairness for FSC as a natural modification of the above problem where the fairness constraint is a proportion of the solution size as opposed to a fixed value.

The set of subsets satisfying fairness constraint above for FSM is not a matroid. However, it was proven by \cite{el2020fairness} that we can convert an instance of FSM into an instance of submodular maximization problem with a matroid constraint. The new constraint system is a matroid, called the \textit{fairness matroid}. This result of \citeauthor{el2020fairness} is stated in the Lemma \ref{lem:prelim} in Appendix \ref{apdx:prelim}. 
We denote the fairness matroid as $\mathcal{M}_{fair}(\mathcal{P},\kappa,\vec{l},\vec{u})=\{S \subseteq U :  |S \cap U_c| \leq u_c, \forall c \in [N], \sum_{c\in [N]}\max\{|S\cap U^c|, l_c\}\leq k\}$, where $\mathcal{P}=\{U_1,...,U_N\}$ is the partition of the ground set $U$, $k$ is the total cardinality constraint, $\vec{l},\vec{u}\in \mathbb{N}^N$ are the lower and upper bound vectors respectively. 
Below we propose the idea of a $\beta$-extension of a fairness matroid, which we will use in our bicriteria algorithms for FSM. 
\begin{definition}
    For any $\beta\in \mathbb{N}_+$, we define the $\beta$-extension of the fairness matroid to be $\mathcal{M}_\beta=\mathcal{M}_{fair}(\mathcal{P},\beta \kappa,\beta\vec{l},\beta\vec{u})=\{S \subseteq U :  |S \cap U_c| \leq \beta u_c, \forall c \in [N], \sum_{c\in [N]}\max\{|S\cap U^c|, \beta l_c\}\leq \beta \kappa\}$.
\end{definition}

Our continuous algorithms will use the multilinear extension of $f$, defined as follows.
\begin{definition} For any submodular objective $f:2^U\rightarrow\mathbb{R}_+$ with $|U|=n$, the multi-linear extension of $f$ is defined as 
$\vect{F}(\vect{x})=\sum_{S\subseteq U}\prod_{i\in S}x_i\prod_{j\notin S}(1-x_j)f(S)$ where $\vect{x}\in[0,1]^n$, and $x_i$ is the $i$-th coordinate of $\vect{x}$. 
If we define $S(\vect{x})$ to be a random set that contains each element $i\in U$ with probability $x_i$, 
then by definition, we have that $\vect{F}(\vect{x})=\mathbb{E} f(S(\vect{x}))$.
\end{definition}

We now present the definitions of discrete and continuous algorithms with an $(\alpha,\beta)$-bicriteria approximation ratio for FSM, which is defined to find $\arg\max_{S\in \mathcal{M}_{fair}(U,k,\vec{l},\vec{u})}f(S)$. 
\begin{definition}
\label{def:bicri-SM}
     A discrete algorithm for FSM with an  $(\alpha,\beta)$-bicriteria approximation ratio returns a solution $X$ such that 
\[f(X)\geq \alpha f(OPT)\qquad\forall c\in[N],\,\, |X\cap U_c|\leq \beta u_c,\qquad\sum_{c\in [N]}\max\{|X\cap U_c|,\beta l_c\}\leq \beta k.\]
Here $OPT$ is the optimal solution of the problem FSM, i.e., $OPT=\arg\max_{S\in \mathcal{M}_{fair}(P,k,\vec{l},\vec{u})}f(S)$. 
\end{definition}
By this definition, we have that an algorithm satisfies a $(\alpha,\beta)$-bicriteria approximation ratio for FSM i.f.f the output set $S$ satisfies $f(S)\geq \alpha f(OPT)$ and that $S\in \mathcal{M}_\beta$.
\begin{definition}
    \label{def:bicri-SM_conti}
     A continuous algorithm with $(\alpha,\beta)$-bicriteria approximation ratio for FSM returns a fractional solution $\vect{x}$ such that 
\[\vect{F}(\vect{x})\geq \alpha f(OPT),\qquad\vect{x}\in \mathcal{P}(\mathcal{M}_\beta).\]
Here $OPT$ is the optimal solution of the problem FSM, i.e., $OPT=\arg\max_{S\in \mathcal{M}_{fair}(P,\kappa,\vec{l},\vec{u})}f(S)$. $\mathcal{M}_\beta$ is the $\beta$-extension of the fairness matroid $\mathcal{M}_{fair}(P,\kappa,\vec{l},\vec{u})$, and $\mathcal{P}(\mathcal{M}_\beta)$ is the matroid polytope of $\mathcal{M}_\beta$. More specifically, $\mathcal{P}(\mathcal{M}_\beta)=conv\{\vect{1}_S:S\in\mathcal{M}_\beta\}$.
\end{definition}

Finally, we define notation that will be used throughout this paper. We use $[N]$ to denote the set $\{1,2,...,N\}$. The marginal gain of adding an element $s$ to the subset $S$ is denoted as $\Delta f(S,s)$. In addition, for any vector $\vec{v}=(v_1,v_2,...,v_N)$, and any $k\in \mathbb{R}$, we define $k\vec{v}=(kv_1,kv_2,...,kv_N)$, and we define $\lceil  \vec{v}\rceil=(\lceil  
 v_1\rceil,\lceil  
 v_2\rceil,...,\lceil  v_N\rceil)$ and $\lfloor  \vec{v}\rfloor=(\lfloor  
 v_1\rfloor,\lfloor  
 v_2\rfloor,...,\lfloor  v_N\rfloor)$.




%% file: sections/convert_thm.tex
\section{Conversion Algorithms for FSC}
\label{sec:conv}
In this section, we introduce two algorithms that make use of the dual relationship between FSC and FSM, and convert bicriteria approximation algorithms for FSM into ones for FSC. The first algorithm, \conv, is designed to convert discrete algorithms for FSM into ones for FSC. In particular, \conv takes an $(\gamma,\beta)$-bicriteria approximation algorithms for FSM that runs in time $\mathcal{T}(n,\kappa)$ and converts it into a $((1+\alpha)\beta,\gamma)$-bicriteria approximation algorithm for FSC that runs in time $O(\sum_{i,(1+\alpha)^{i-1}\leq|OPT|})\mathcal{T}(n,(1+\alpha)^i))$. However, because of the matroid constraint, better approximation guarantees for FSM may be achieved by a continuous algorithm that produce a fractional solution. Motivated by this, our second converting algorithm, \convc, takes a continuous $(\gamma,\beta)$-bicriteria approximation algorithm (where guarantees are with respect to the multilinear extension as described in Section \ref{sec:problem_setup}), and converts it into a $((1+\alpha)\beta,\frac{(1-\frac{\varepsilon}{2})\gamma-\frac{\varepsilon}{3}}{1+\frac{\varepsilon}{2}+\frac{\varepsilon}{3\gamma}})$-bicriteria approximation algorithm for FSC. In the next section, we will develop corresponding bicriteria approximation algorithms for FSM that can be used along with the results in this section in order to produce approximately optimal solutions for FSC that are arbitrarily close to being feasible.

We first consider our algorithm \conv for converting discrete algorithms for FSM. Pseudocode for \conv is provided in Algorithm \ref{alg:conver}. \conv takes as input an instance of FSC, a $(\gamma,\beta)$-bicriteria approximation algorithm for FSM, and a parameter $\alpha>0$. Each iteration of the while loop from Line \ref{Line:alg_convert,loop_start} to Line \ref{Line:alg_convert,loop_ends} corresponds to a guess $\kappa$ on the size of the optimal solution to the instance of FSC. For each guess $\kappa$, we have an instance of FSM with budget $\kappa$, fairness vector of lower bound $\kappa\vec{p}$, fairness vector of upper bound $\kappa\vec{q}$.  
We then run the algorithm for FSM on this instance to get a set $S$. Notice that this algorithm will convert the matroid corresponding to the instance of FSM into its $\beta$-extension. In Lines \ref{line:rounding_starts} to \ref{line:rounding_ends}, \conv adds additional elements so that the lower bounds are met for every one of the partitions. Next in Lines \ref{line:rounding2_starts} to \ref{line:rounding2_ends}, \conv then adds elements until the size constraint $\beta\kappa$ is met, without breaking the fairness constraints. Finally, \conv checks if the set $S$ satisfies $f(S)\geq\gamma\tau$. If it does not, the guess of optimal solution size increases by a factor of $1+\alpha$ and the process repeats itself.

\begin{algorithm}[t!]
\caption{\texttt{convert-fair}}\label{alg:conver}
\textbf{Input}: An FSC instance with threshold $\tau$, fairness parameters $\vec{p}$, $\vec{q}$, partition of $U$ $\mathcal{P}$, a $(\gamma,\beta)$-bicriteria approximation algorithm for FSM, $\alpha>0$\\
\textbf{Output}: $S\subseteq U$
\begin{algorithmic}[1]
\State $\kappa\gets(1+\alpha)$, $S\gets\emptyset$.
\While{$f(S)<\gamma\tau$}\label{Line:alg_convert,loop_start} 
\State $S\gets$ Run $(\gamma,\beta)$-approximation algorithm for FSM with fairness matroid $\mathcal{M}_{fair}(\mathcal{P},\kappa,\lfloor\vec{p}\kappa\rfloor,\lceil\vec{q}\kappa\rceil)$
\label{line:convmax}

\State $\kappa\gets\lceil(1+\alpha)\kappa\rceil$
 
\label{Line:alg_convert,loop_ends}
\State \textcolor{gray}{//Rounding the solution}

 \For {$c\in[N]$} \label{line:rounding_starts}
        \If{$|S\cap U_c| < \beta\lfloor p_c\kappa\rfloor$} 
        \State Add new elements from $U_c/S$ to $S$ 
        until $|S\cap U_c| \geq \beta\lfloor p_c\kappa\rfloor$
        \EndIf
\EndFor\label{line:rounding_ends}
        \If{$|S| < \beta\kappa$} 
        \For {$c\in[N]$} \label{line:rounding2_starts}
        \While{$|S| < \beta\kappa$ and $|S\cap U_c|< \beta\lceil q_c\kappa\rceil$}
        \State Add new elements in $U_c/S$ to $S$
        \EndWhile
        \EndFor\label{line:rounding2_ends}
        \EndIf
        \EndWhile
\State \textbf{return} $S$
\end{algorithmic}
\end{algorithm}


We now state the theoretical results for \conv below in Theorem \ref{alg:conver}. Notice that Theorem \ref{alg:conver} assumes that the instance of FSC satisfies that $\sum_{c\in[N]}\min\{q_c, \frac{|U_c|}{\beta(1+\alpha)|OPT|)}\}\geq1$. Recall from the definition of FSC that it is already assumed $\sum_{c\in[N]}q_c\geq 1$, so this assumption is essentially requiring that there be enough elements within each set $U_c$ of the partition so that we can run the bicriteria algorithm for FSM. We defer the proof and analysis of the proof of Theorem~\ref{thm:convert} to Appendix \ref{appdx:conv}.

\begin{restatable}{theorem}{thmconvert}
    \label{thm:convert}
   Assuming $\sum_{c\in[N]}\min\{q_c, \frac{|U_c|}{\beta(1+\alpha)|OPT|)}\}\geq1$, any $(\gamma,\beta)$-bicriteria approximation algorithm for FSM that returns a solution set in time $\mathcal{T}(n,\kappa)$ can be converted into an approximation algorithm for FSC that is a $((1+\alpha)\beta,\gamma)$-bicriteria approximation algorithm that runs in time $O(\sum_{i=1}^{\frac{\log(|OPT|)}{\log(\alpha+1)}}\mathcal{T}(n,(1+\alpha)^i))$.
\end{restatable}

We now present our algorithm \convc for converting continuous algorithms. To motivate it, notice that here we can't directly use the converting theorem for discrete algorithms. There are two reasons for this: (i) The output solution is fractional so we need a rounding step, and (ii) the bicriteria approximation ratio for the continuous algorithms is on the value of the multi-linear extension and we don't have exact access to $F$, so we can't check directly if $\vect{F}(\vect{x})\geq\gamma\tau$ as we did on Line \ref{Line:alg_convert,loop_start} of \conv. Then we develop the converting algorithm \convc for continuous algorithms. The key idea of the converting theorem is similar to \conv, so we defer the pseudocode of \convc to Algorithm \ref{alg:conv_c} in Section \ref{appdx:continuous} of the appendix. Here we describe the major differences. For each guess of optimal solution size $\kappa$, \convc invokes the continuous subroutine algorithm for FSM to obtain a fractional solution $\vect{x}$. Since $\vect{F}$ can't be queried exactly in general, we estimate $\vect{F}(\vect{x})$ by taking a sufficient number of samples in Line \ref{line:estimateY}. Once the estimate of $\vect{F}(\vect{x})$ is higher than $\gamma\tau$, we use the pipage rounding technique to convert $\vect{x}$ into a subset $S$, and then use the rounding procedure analogous to that in Lines \ref{line:rounding_starts} to \ref{line:rounding2_ends} in Algorithm \ref{alg:conver} to obtain a solution set with fairness guarantee. Notice that since the solution set obtained from the pipage rounding step is only guaranteed to satisfy that $\mathbb{E}f(S)\geq\vect{F}(\vect{x})$, the approximation guarantee on the function value in Theorem \ref{thm:continuous} holds in expectation. The corresponding theoretical guarantees for \convc are stated below in Theorem \ref{thm:conv_continuous}. 
The proof of Theorem \ref{thm:conv_continuous} can be found in Section \ref{appdx:continuous} of the appendix.
\begin{theorem}
\label{thm:conv_continuous}
   Assuming $\sum_{c\in[N]}\min\{q_c, \frac{|U_c|}{\beta(1+\alpha)|OPT|)}\}\geq1$, with probability at least $1-\delta$, any $(\gamma,\beta)$-bicriteria approximation algorithm for FSM that returns a solution set in time $\mathcal{T}(n,\kappa)$ with probability at least $1-\frac{\delta}{n}$ can be converted into an approximation algorithm for FSC that is a $((1+\alpha)\beta,\frac{(1-\frac{\varepsilon}{2})\gamma-\frac{\varepsilon}{3}}{1+\frac{\varepsilon}{2}+\frac{\varepsilon}{3\gamma}})$-bicriteria approximation algorithm where $\frac{(1-\frac{\varepsilon}{2})\gamma-\frac{\varepsilon}{3}}{1+\frac{\varepsilon}{2}+\frac{\varepsilon}{3\gamma}} $ holds in expectation. The query complexity is at most $O\left(\sum_{i=1}^{\frac{\log(|OPT|)}{\log(\alpha+1)}}\mathcal{T}(n,(1+\alpha)^i)+\frac{n\log_{1+\alpha}|OPT|}{\varepsilon^2}\log\frac{n}{\delta}\right)$.
\end{theorem}

%% file: sections/bicriteria_alg.tex
\section{Bicriteria Algorithms for FSM}
\label{sec:alg_for_FSM}
In the last section, we propose converting algorithms to convert bicriteria algorithms for FSM into ones for FSC. Existing algorithms for FSM can be used as the input $(\gamma,\beta)$-bicriteria subroutine, but these algorithms all return feasible solutions to the instance of FSM and have guarantees of $\beta=1$ and $\gamma\leq 1-1/e$. After applying the converting algorithms in Section \ref{sec:conv} this results in solutions for FSC that are far from feasible. For example, the greedy algorithm for FSM proposed has $\gamma=1/2$ and $\beta=1$, and the continuous greedy algorithm for FSM has $\gamma=1-1/e$ and $\beta=1$ \citep{CelisHV18}.

In this section, we propose new algorithms that can be used for FSM where $\gamma$ is arbitrarily close to 1 and $\beta > 1$. As a result, the algorithms proposed in this section can be paired with the converting theorems in Section \ref{sec:conv} to find solutions to our instance of FSC that are arbitrarily close to being feasible. We propose three bicriteria algorithms for FSM. The first two algorithms are the discrete algorithms \texttt{greedy-fairness-bi} and \threalglong, which both achieve bicriteria approximation ratios of $(1-O(\varepsilon),\frac{1}{\varepsilon})$ where the former is in time $O(n\kappa/\epsilon)$ and the latter in time $O(\frac{n}{\varepsilon}\log\frac{\kappa}{\varepsilon})$. The third algorithm is a continuous one, \contialglong, which achieves a $(1-O(\varepsilon),\ln\frac{1}{\varepsilon}+1)$ bicriteria approximation ratio in time $O\left(\frac{n\kappa}{\varepsilon^4}\ln^2(n/\varepsilon)\right)$.

In the case of submodular maximization with a cardinality constraint without fairness, one can find a solution with $f$ value that is a factor of $1-\varepsilon$ off of that of the optimal solution by greedily adding $O(\ln(1/\varepsilon))\kappa$ elements beyond the cardinality constraint $\kappa$. However, existing algorithms for FSM transform the instance into an instance of submodular maximization subject to a fairness matroid constraint, and it is not clear how one can take an analogous approach and produce an infeasible solution in order to get a better approximation guarantee when dealing with a matroid constraint while maintaining a fair solution. We propose the $\beta$-extension of a fairness matroid, defined in Section \ref{sec:problem_setup}, in order to get a $(\gamma,\beta)$-bicriteria algorithm for FSM with $\gamma > 1-1/e$. In particular, we will return a solution that is a feasible solution to the  $\beta$-extension of the fairness matroid corresponding to our instance of FSM. 

We now introduce two lemmas concerning the $\beta$-extension of a matroid that will be needed for our algorithms and their theoretical analysis. Before we proceed to present our algorithms, we first introduce the following general lemma that helps to build a connection between the fairness matroid $\mathcal{M}$ and its $\beta$-extension $\mathcal{M}_\beta$.
For the sake of simplicity in notation throughout this section and its subsequent proofs, we use the notation $\mathcal{M}_\beta$, representing the $\beta$-extension of $\mathcal{M}_{fair}(P,\kappa,\vec{l},\vec{u})$, which is defined in Section \ref{sec:problem_setup}. Since the bicriteria algorithm for FSM is designed as a subroutine for the converting theorem with input $\mathcal{M}_{fair}(P,\kappa,\vec{p}\kappa,\vec{q}\kappa)$, we have that here $l_c=\lfloor p_c\kappa\rfloor$ and $u_c=\lceil q_c\kappa\rceil$. From the fact that $\sum_{c\in[N]}p_c\leq 1\leq \sum_{c\in[N]}q_c$, we have that $\sum_{c\in [N]}l_c\leq\kappa\leq\sum_{c\in[N]}u_c $. Since $\sum_{c\in[N]}\min\{q_c, \frac{|U_c|}{\beta(1+\alpha)|OPT|)}\}\geq1$, we have that $rank(\mathcal{M}_{fair}(P,\kappa,\vec{l},\vec{u}))=\kappa$ and that $rank(\mathcal{M}_\beta)=\beta\kappa$. Therefore, we have the following Lemma \ref{lem:assump_for_SM}. 
 

\begin{lemma}
\label{lem:assump_for_SM}
$\sum_{c\in[N]}l_c\leq\kappa\leq\sum_{c\in[N]}u_c$ and $rank(\mathcal{M}_{fair}(P,\kappa,\vec{l},\vec{u}))=\kappa$, $rank(\mathcal{M}_\beta)=\beta\kappa$.
\end{lemma}
\begin{lemma}
   \label{lem:feasible_OPT} 
    For any $\beta\in\mathbb{N}_+$ and any fairness matroid $\mathcal{M}_{fair}(P,\kappa,\vec{l},\vec{u})$, denote $\mathcal{M}_{\beta}$ as the $\beta$-extended fairness matroid of $\mathcal{M}_{fair}(P,\kappa,\vec{l},\vec{u})$. Then for any set $S\in\mathcal{M}_{\beta}$ with $|S|=\beta \kappa$, $T\in \mathcal{M}_{fair}(P,\kappa,\vec{l},\vec{u})$ with $|T|=\kappa$, and any permutation of $S=(s_1,s_2,...,s_{\beta\kappa})$, there exist a sequence $E=(e_1,e_2,...,e_{\beta\kappa})$ such that each element in $T$ appears $\beta$ times in $E$ and that 
    \begin{align*}
        S_i\cup \{e_{i+1}\}\in \mathcal{M}_\beta , \qquad\forall i\in\{0,1,...,\beta\kappa\}
    \end{align*}
    where $S_i=(s_1,s_2,...,s_i)$ and $S_0=\emptyset$.
\end{lemma}
It is worth noting that this lemma reveals an important fact about the fairness matroid: For each base set $T\in \mathcal{M}_{fair}(P,\kappa,\vec{l},\vec{u})$, and each subset $S\in \mathcal{M}_{\beta}$ that is a base set, we can find a mapping from $T$ to a sequence $E$ that contains $\beta$ copies of $T$ such that $S_i\cup \{e_i\}$ is always feasible for $\mathcal{M}_{\beta}$.  The proof of this lemma is deferred to the appendix. Notice that since $\mathcal{M}_\beta$ is a matroid, then for any subset $S_1\subseteq S_2$, if $S_2\in\mathcal{M}_\beta$ then $S_1\in\mathcal{M}_\beta$. Consequently, the above lemma holds for not only just base set of $\mathcal{M}_\beta$, but also for any subset of $\mathcal{M}_\beta$ by adding dummy variables to the end of the sequence $S$ if the number of elements in $S$ is less than $\beta\kappa$. Building upon this lemma, we propose three bicriteria algorithms in the following part.

\subsection{Discrete Bicriteria Algorithms for FSM}
\label{sec:discrete}
We now analyze two discrete bicriteria algorithms for FSM, \greedyalg and \threalglong. Let SMMC refer to the problem of monotone submodular maximization with a matroid constraint.
\greedyalg is based on the standard greedy algorithm which is well-known to produce a feasible solution with a $1/2$ approximation guarantee in $O(nk)$ time for SMMC \citep{fisher1978analysis}, where $k$ is the rank of the matroid. \greedyalg proceeds in a series of rounds, where at each round we select the element $x\in U$ with the highest marginal gain to $f$ that stays on the $1/\varepsilon$-extension of the fairness matroid corresponding to the instance of FSM, i.e. $S\cup\{x\}\in\mathcal{M}_{1/\varepsilon}$.
\threalglong is based on the threshold greedy algorithm \citep{badanidiyuru2014fast}, which is also a $1/2 - \varepsilon$ approximation for SMMC but requires only $O(n \log k)$ queries of $f$. \threalglong iteratively makes passes through the universe $U$ and adds all elements into its solution with marginal gains exceeding $\tau$ that are feasible with respect to the $1/\varepsilon$-extension of the fairness matroid, and this threshold is decreased by $1 - \varepsilon$ at each round until it falls below a stopping criterion.
Notice that these algorithms specifically use the $\beta$-extension of the fairness matroid, and therefore do not apply to the more general setting of submodular maximization with a matroid constraint.
Pseudocode for the algorithms \greedyalg and \threalglong are included in Appendix \ref{appdx:alg_for_FSM} as Algorithms \ref{alg:fairness-bi} and Algorithm \ref{alg:thres-fairness-bi} respectively.

We now present the theoretical guarantees of \greedyalg and \threalglong. The key benefit of these algorithms over existing ones for FSM is that by making $\varepsilon$ arbitrarily small and using \conv in Section \ref{sec:conv}, we have algorithms for FSC that are arbitrarily close to being feasible. In particular, if we use \greedyalg as a subroutine in \conv, we have a $(\frac{1}{\varepsilon}+1, 1-\varepsilon)$-bicriteria algorithms for FSC in $\mathcal{O}(n\log(n)\kappa/\epsilon)$ queries of $f$. If we use \threalglong, we get a similar approximation guarantee in $\mathcal{O}(n\log(n)\log(\kappa/\epsilon)/\epsilon^2)$ queries of $f$. The proofs of both of these theorems can be found in Section \ref{appdx:alg_for_FSM} of the appendix. 

\begin{restatable}{theorem}{thmgreedy}
    \label{thm:greedy}
    Suppose that \texttt{greedy-fairness-bi} is run for an instance of FSM, 
   then \texttt{greedy-fairness-bi}
   outputs a solution $S$ that satisfies a $(1-\varepsilon, \frac{1}{\varepsilon})$-bicriteria approximation guarantee in at most $O\left(n\kappa /\varepsilon\right)$ queries of $f$.
\end{restatable}
\begin{restatable}{theorem}{thmthresholdfairnessbi}
\label{theorem:threshold-fairness-bi}
Suppose that \texttt{threshold-fairness-bi} is run for an instance of FSM with $\varepsilon \in (0,1)$. Then \texttt{threshold-fairness-bi} outputs a solution $S$ that satisfies a $(1 - 2\varepsilon, \frac{1}{\varepsilon})$-bicriteria approximation guarantee in at most $O\left(n/\varepsilon \log(\kappa /\varepsilon\right))$ queries of $f$.
\end{restatable}

%% file: sections/continuousTG.tex
\subsection{Continuous Algorithms for FSM}
\label{sec:continuous}
A downside to the discrete greedy algorithms proposed in Section \ref{sec:alg_for_FSM} is that we are above our budget $\kappa$ by a factor of $1/\varepsilon$, which is weaker than the analogous guarantee of $\ln(1/\varepsilon)$ that the greedy algorithm gives for submodular maximization with a cardinality constraint without fairness.
We now introduce and analyze our continuous algorithm \contialglong (\contialg), which produces a fractional solution for FSM that achieves a $(1-O(\varepsilon), \ln(1/\varepsilon)+1)$-bicriteria approximation ratio in $O\left(n\kappa\ln^2(n)\right)$ time. \contialg is based on the continuous threshold greedy algorithm of \cite{badanidiyuru2014fast}. Compared to the discrete algorithms presented in Section \ref{sec:discrete}, \contialg improves the ratio on the cardinality of the solution from $O(1/\varepsilon)$ to $O(\ln(1/\varepsilon))$, and therefore has as strong of guarantees as the greedy algorithm without fairness. 
We can achieve a discrete solution with an arbitrarily small loss by employing rounding schemes, like swap rounding~\citep{swaprounding}, on the returned fractional solution $\mathbf{x}$.

\contialg iteratively takes a step of size $\varepsilon$ in the direction $\mathbf{1}_{B}$, where $\mathbf{1}_B$ is the indicator function of a set $B\subseteq U$, over $1/\varepsilon$ iterations. At each step, the set $B$ is determined by the subroutine \contisublong (\contisub). \contisub builds $B$ over a series of rounds corresponding to thresholds $w$, where $w$ begins as the max singleton marginal gain and the rounds exit once $w$ is sufficiently small. During each round, we iterate over the universe $U$, and if an element $u\in U$ can be added to $B$ while staying on the $\ln(1/\epsilon)+1$-extension of the fairness matroid, then we approximate the multilinear extension $f$ and add $u$ to $B$ if and only if the marginal gain is above $w$. 
Pseudocode for \contialg is provided in Algorithm \ref{alg:CTG}, and pseudocode for \contisublong{} is provided in Algorithm \ref{alg:CCTG_subroutine}. 
\begin{algorithm}[t]
\caption{\contialglong(\contialg)}\label{alg:CTG}
 \begin{algorithmic}[1]
 \State \textbf{Input:} $\varepsilon$, $\delta$, $\mathcal{M}\in 2^U$
 \State $\vect{x}\gets \vect{0}$
  \State $d:=\max_{s\in \mathcal{M}}{f}(s)$, 
 
 \For{$t=1$ to $1/\varepsilon$}
 \State $B\gets$\contisublong($\vect{x}$, $\varepsilon$, $\delta$, $d$, $\mathcal{M}$)
 \State $\vect{x}\gets \vect{x}+\varepsilon\cdot\vect{1}_{B}$
 \EndFor
 \State \textbf{return} $\vect{x}$
 \end{algorithmic}
\end{algorithm}

\begin{algorithm}[t]
\caption{\contisublong(\contisub)}\label{alg:CCTG_subroutine}
 \begin{algorithmic}[1]
 \State \textbf{Input:} $\vect{x}$, $\varepsilon$, $\delta$, $d$, $\mathcal{M}\in 2^U$
 \State $B\gets \emptyset$
 \State Denote $\mathcal{M}_{fair}(P,\kappa\ln(1/\varepsilon),\vec{l}\ln(1/\varepsilon),\vec{u}\ln(1/\varepsilon))$ as $\mathcal{M}_{\ln(1/\varepsilon)}$.
 \For{$w=d$; $w>\frac{\varepsilon d}{\kappa}$; $w=w(1-\varepsilon)$}
 \For{$u\in U$}
 \State $X={\Delta f}(S(\vect{x}+\varepsilon\vect{1}_B),u)$

\If{$B\cup\{u\}\in\mathcal{M}_{\ln(1/\varepsilon)+1}$}
\State $\hat{X}\gets $ average over  $\frac{3\kappa}{\varepsilon^2}\log{\frac{4n^4}{\varepsilon^3}}$ samples from $\mathcal{D}_X$
 \If{$\hat{X}\geq w$}
 \State $B\gets B\cup \{u\}$
 \EndIf
  \EndIf
 \EndFor

 \State $w=w(1-\varepsilon)$
 \EndFor
 \State \textbf{return} $B$ \end{algorithmic}
\end{algorithm}

\begin{restatable}{theorem}{thmcontinuous}
\label{thm:continuous}
     Suppose that Algorithm \ref{alg:CTG} is run for an instance of FSM, then with probability at least $1-\frac{1}{n^2}$, \contialglong
   outputs a solution $S$ that satisfies a $(1-7\varepsilon,\ln(\frac{1}{\varepsilon})+1)$-bicriteria approximation guarantee in at most $O\left(\frac{n\kappa}{\varepsilon^4}\ln^2(n/\varepsilon)\right)$ queries of $f$.
\end{restatable}
Then, by applying a converting theorem, we can obtain the algorithm for submodular cover that achieves an approximation ratio of $((1+\alpha)(\ln(\frac{1}{\varepsilon})+1),1-O(\varepsilon))$, which aligns with the best-known results for bicriteria submodular cover without the fairness constraint \cite{chen2024bicriteria, iyer2013submodular}. The detailed theoretical guarantee and proof of the algorithm can be found in Corollary \ref{coro:combine_convert_with_conti} in the appendix.


%% file: sections/experiments.tex
\section{Experiments} \label{sec:exp}
In this section, we evaluate several of our algorithms for FSC on instances of fair maximum coverage, where the objective is to identify a set of fixed nodes that optimally maximize coverage within a graph. The dataset utilized is a subset of the Twitch Gamers dataset~\citep{twitch}, comprising 5,000 vertices (users) who speak English, German, French, Spanish, Russian, or Chinese. We aim to develop a solution with a high $f$ value exceeding a given threshold $\tau$ while ensuring a fair balance between users who speak different languages. Additional discussion about the application as well as experimental setup including parameter settings are included in Appendix \ref{apdx:add-exp}. In addition, we include experiments on instances of fair image summarization in Appendix \ref{apdx:add-exp}.

We evaluate the performance of our discrete bicriteria algorithms \greedyalg and \threalglong for FSM as subroutines in our algorithm \conv. In addition, we consider the baseline algorithm, \texttt{greedy-bi}, which is the standard greedy algorithm for submodular cover without fairness. Figures \ref{fig:radar-greedy-bi}, \ref{fig:radar-greedy-fairness-bi} and \ref{fig:radar-threshold-fairness-bi} showcase the distribution of users speaking different languages in the solutions produced by these algorithms with $\tau = 2400$. Figures \ref{fig:max-cover-tau-f}, \ref{fig:max-cover-tau-cost} and \ref{fig:max-cover-tau-diff} present the performance of these algorithms ($f$ value, cost, and fairness difference) for varying values of $\tau$. As shown in Figure~\ref{fig:radar-greedy-bi}, with $\tau = 2400$, over $80 \%$ of the users in the solution returned by \texttt{greedy-bi} are English speakers, which indicates a lack of fairness in user language distribution. While the solutions produced by \greedyalg{} and \threalglong{} exhibit significantly fairer distributions across different languages, demonstrating the effectiveness of our proposed algorithms. Further, as the value of given $\tau$ increases, the magnitude of this difference also increases (see Figure \ref{fig:max-cover-tau-diff}). Figure \ref{fig:max-cover-tau-f} showcases that for all these algorithms the objective function value $f(S)$ scales almost linearly with the threshold $\tau$, which aligns with the theoretical guarantees of the approximation ratio. Additionally, as shown in Figure~\ref{fig:max-cover-tau-cost}, the cost of the solutions returned by our proposed algorithms is higher than that of the solution from \texttt{greedy-bi}. This is an expected trade-off, as our algorithms have to include more elements to maintain the approximation ratio while ensuring fairness. Overall, our proposed algorithms are efficient and effective in producing a fair solution.





\begin{figure}[t]
    \centering
    \begin{subfigure}{0.33\textwidth}
      \includegraphics[width=1\linewidth]{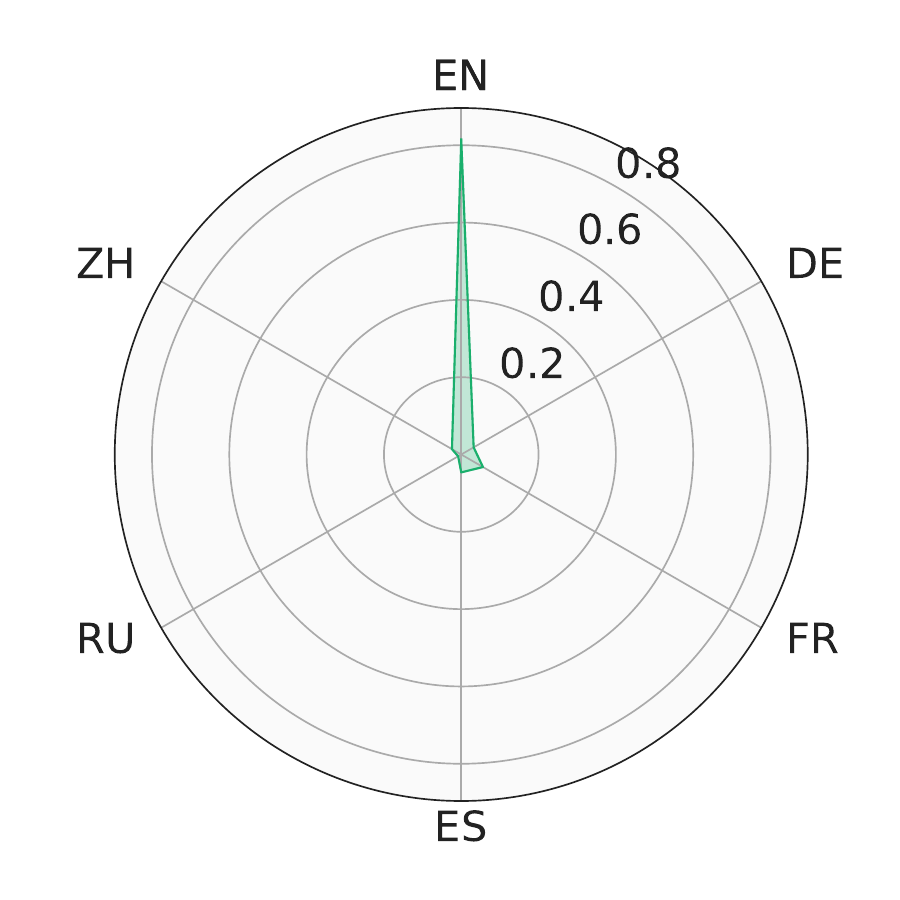}
      \caption{\texttt{greedy-bi}}
      \label{fig:radar-greedy-bi}
    \end{subfigure}%
    \begin{subfigure}{0.33\textwidth}
      \includegraphics[width=1\linewidth]{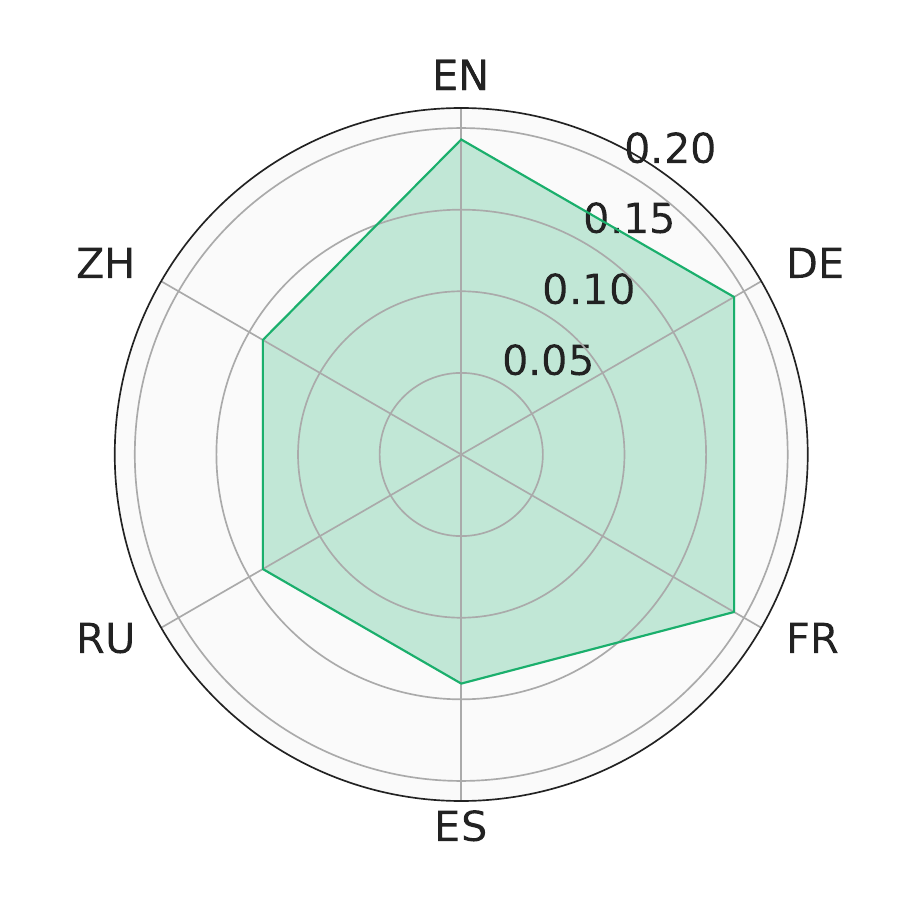}
      \caption{\texttt{greedy-fairness-bi}}
      \label{fig:radar-greedy-fairness-bi}
    \end{subfigure}
    \begin{subfigure}{0.33\textwidth}
      \includegraphics[width=1\linewidth]{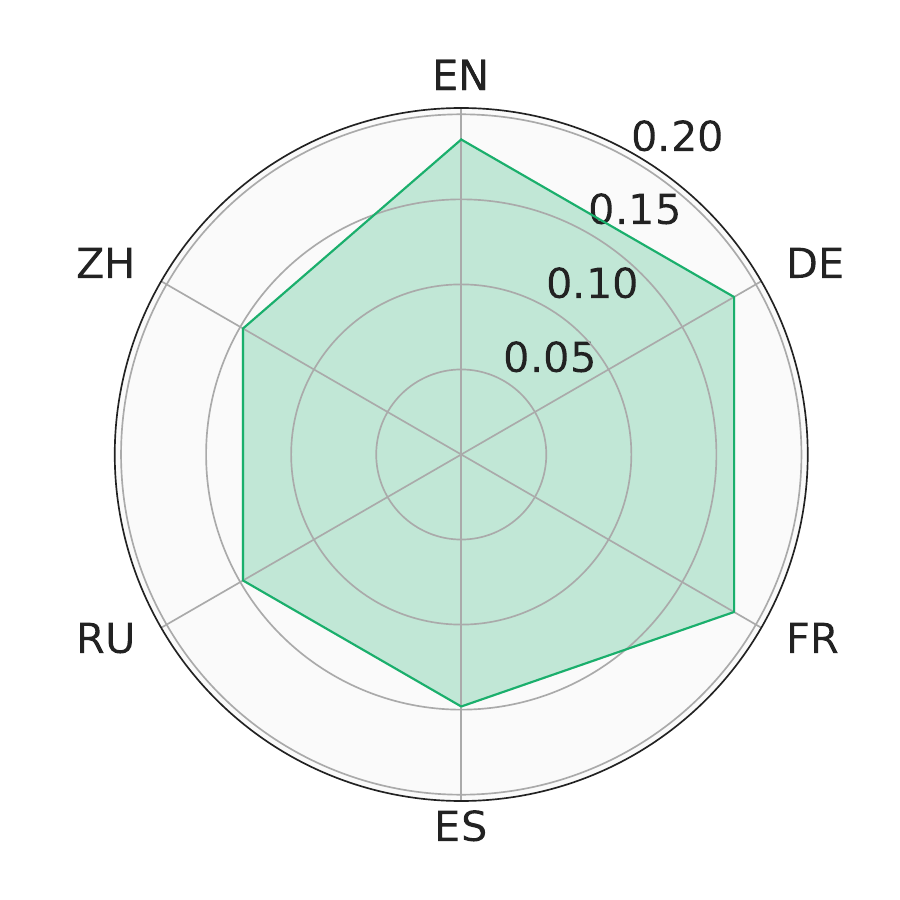}
      \caption{\texttt{threshold-fairness-bi}}
      \label{fig:radar-threshold-fairness-bi}
    \end{subfigure}

    \begin{subfigure}{0.33\textwidth}
      \includegraphics[width=1\linewidth]{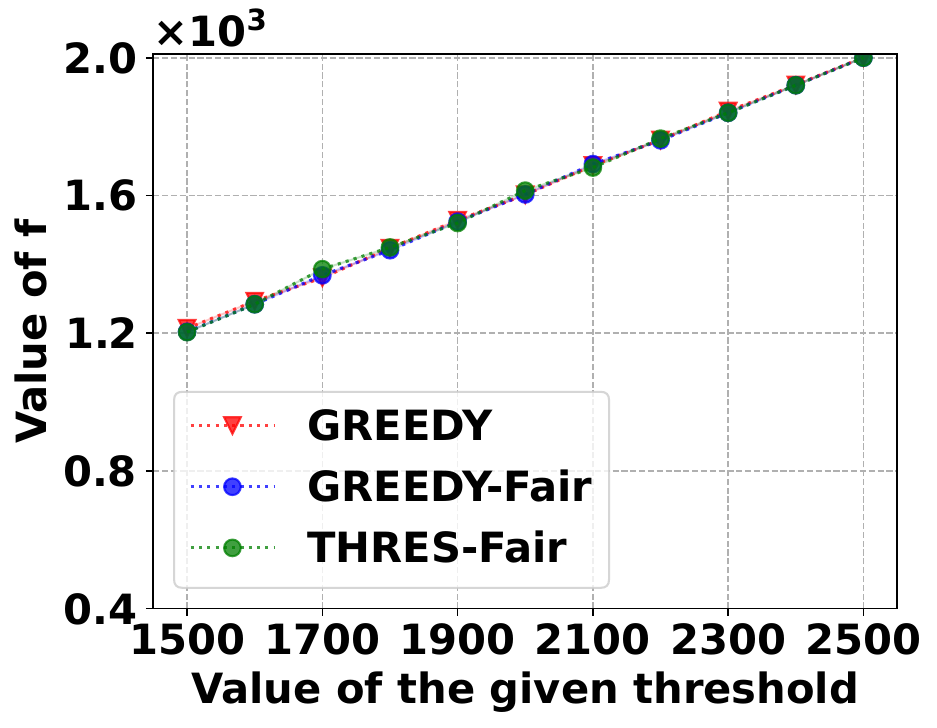}
      \caption{$f$}
      \label{fig:max-cover-tau-f}
    \end{subfigure}%
    \begin{subfigure}{0.33\textwidth}
      \includegraphics[width=1\linewidth]{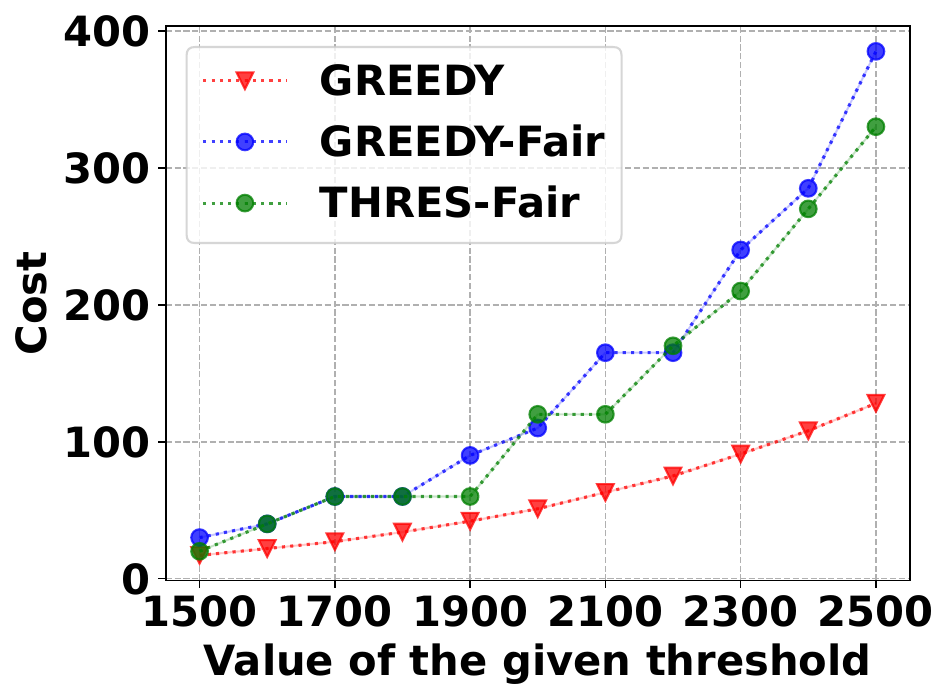}
      \caption{Cost}
      \label{fig:max-cover-tau-cost}
    \end{subfigure}
    \begin{subfigure}{0.33\textwidth}
      \includegraphics[width=1\linewidth]{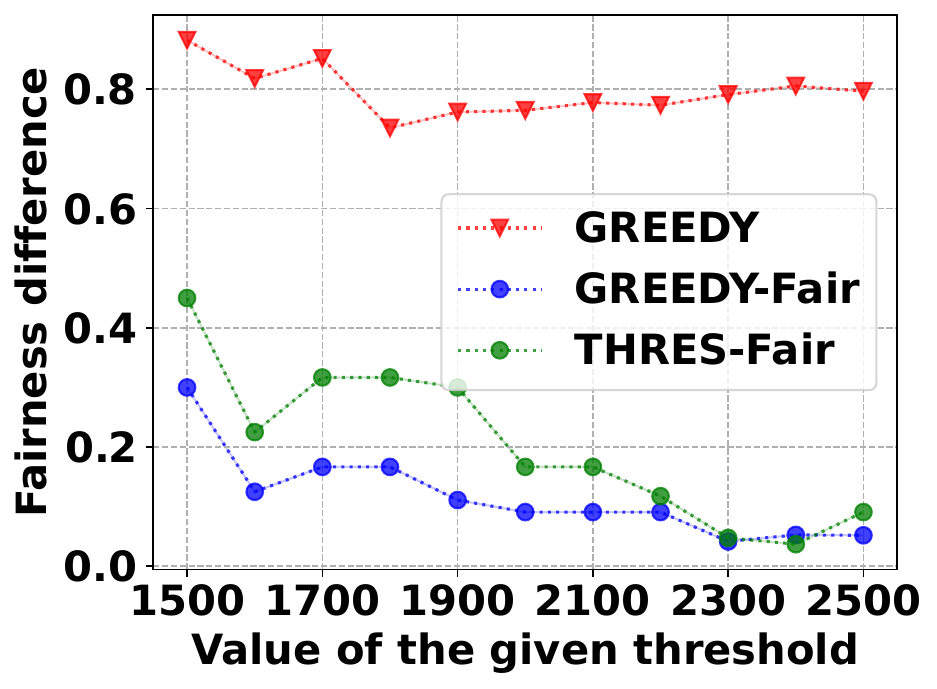}
      \caption{Fairness difference}
      \label{fig:max-cover-tau-diff}
    \end{subfigure}

    \caption{Performance comparison on the Twitch\_5000 dataset for Maximum Coverage. \ref{fig:radar-greedy-bi}, \ref{fig:radar-greedy-fairness-bi}, \ref{fig:radar-threshold-fairness-bi} illustrate the distribution of users speaking different languages in the solutions produced by various algorithms with $\tau = 2400$. $f$: the value of the objective submodular function. Cost: the size of the returned solution. Fairness difference: $(\max_c |S \cap U_c| - \min_c |S \cap U_c|) / |S|$.}
    \label{fig:image-sum-central-exp} 
\end{figure}


%% file: sections/appendix.tex
\appendix
%

\section{Omitted Lemma of Section \ref{sec:problem_setup}}\label{apdx:prelim}
\begin{lemma}[\citep{el2020fairness}]\label{lem:prelim}
\label{lem:fairness_to_matroid}
    If $f$ is monotone, then solving FSM is equivalent to the problem below.
      \begin{align*}
        &\max_{S\in U} f(S)\nonumber\\
        s.t.\qquad& |S\cap U_c|\leq u_c\qquad \forall c\in[N]\nonumber\\
        &\sum_{c\in [N]}\max\{|S\cap U_c|, l_c\}\leq k
    \end{align*}
\end{lemma}


\section{Appendix for Section \ref{sec:conv}}
In this section, we present missing discussions and proofs from Section \ref{sec:conv} in the main paper. We first present missing proofs of Theorem \ref{thm:convert} about algorithm \conv in Section \ref{appdx:conv}. Then we present the proof of Theorem \ref{thm:conv_continuous} about the converting algorithm \convc for continuous algorithms in Section \ref{appdx:conv_conti}. In addition, pseudocode for the algorithm \convc is presented in Algorithm \ref{alg:conv_c}. 

\subsection{Proof of Theorem \ref{thm:convert}}
\label{appdx:conv}

In this section, we present the missing proofs of the lemmas that are used in the proof of Theorem \ref{thm:convert}. In order to prove Theorem \ref{thm:convert}, we need the following two lemmas. Lemma \ref{lem:fairness_for_cover} guarantees that the solution set $S$ after the rounding step satisfies the fairness constraint for cover. Lemma \ref{lem:matroid_subset_relationship} implies the inclusion relationship of the fairness matroid with the same fairness ratios. 


\begin{lemma}
\label{lem:fairness_for_cover}
    For each guess $\kappa$ such that $\kappa\leq(1+\alpha)|OPT|$, the solution set $S$ in Algorithm \ref{alg:conver} satisfies 
\[\beta\lfloor \frac{p_c|S|}{\beta}\rfloor\leq |S\cap U_c|\leq  \beta\lceil \frac{q_c|S|}{\beta}\rceil,\qquad |S|=\beta\kappa.\]
\end{lemma}

\begin{proof}
    Here we denote the solution set returned by the bicriteria algorithm for FSM as $S'$, and the solution set after the rounding steps from Line \ref{line:rounding_starts} to Line \ref{line:rounding_ends} as $S''$. From the definition of bicriteria approximation algorithm for FSM, we can see that the solution set returned by the subroutine algorithm for FSM satisfies that
    \begin{align*}
        &|S'\cap U_c|\leq \beta \lceil q_c\kappa\rceil\\
        &\sum_{c\in[N]}\max\{|S'\cap U_c|,\beta \lfloor p_c\kappa\rfloor\}\leq \beta\kappa
    \end{align*}
    After the rounding steps for each group from Line \ref{line:rounding_starts} to Line \ref{line:rounding_ends}, the solution set satisfies that $ |S''\cap U_c|=\max\{\beta \lfloor p_c\kappa\rfloor,|S'\cap U_c|\}$ for any $c\in[N]$. It then follows that $\beta \lfloor p_c\kappa\rfloor\leq|S''\cap U_c|\leq \beta \lceil q_c\kappa\rceil$. Since that $\sum_{c\in[N]}\max\{|S'\cap U_c|,\beta \lfloor p_c\kappa\rfloor\}\leq \beta \kappa$, we have that 
    \begin{align*}
        |S''|=\sum_{c\in[N]}|S''\cap U_c|=\sum_{c\in [N]}\max\{|S'\cap U_c|,\beta \lfloor p_c\kappa\rfloor\}\leq \beta \kappa.
    \end{align*}
    From the assumption that $\sumN q_c\geq 1$ and $\sum_{c\in[N]}\min\{q_c, \frac{|U_c|}{\beta(1+\alpha)|OPT|)}\}\geq1$, 
after the second rounding phase from Line \ref{line:rounding2_starts} to Line \ref{line:rounding2_ends}, we have $|S|=\beta\kappa$ and that for each group $c$, $$\beta \lfloor p_c\kappa\rfloor\leq|S\cap U_c|\leq \beta \lceil q_c\kappa\rceil.$$ Since the solution set $S$ is of cardinality $\beta\kappa$, then we have 
    \begin{align*}
        \beta \lfloor \frac{p_c|S|}{\beta}\rfloor\leq|S\cap U_c|\leq \beta \lceil \frac{q_c|S|}{\beta}\rceil.
    \end{align*}
\end{proof}


\begin{lemma}
\label{lem:matroid_subset_relationship}
    For any positive integers $\kappa_1,\kappa_2$ such that $\kappa_1\leq\kappa_2$, we have that 
    \begin{align*}
        \mathcal{M}_{fair}(P,\kappa_1,\lceil \vec{p}\kappa_1\rceil,\lceil\vec{q}\kappa_1\rceil)\subseteq \mathcal{M}_{fair}(P,\kappa_2,\lfloor \vec{p}\kappa_2\rfloor,\lceil\vec{q}\kappa_2\rceil)
    \end{align*}
\end{lemma}

\begin{proof}
       The lemma is equivalent to prove that for any subset $A\in\mathcal{M}_{fair}(P,\kappa_1,\lceil \vec{p}\kappa_1\rceil,\lceil\vec{q}\kappa_1\rceil)$, we have that $A$ is also in $\mathcal{M}_{fair}(P,\kappa_2,\lfloor \vec{p}\kappa_2\rfloor,\lceil\vec{q}\kappa_2\rceil)$. Since $\kappa_1\leq\kappa_2$, $|A\cap U_c|\leq \lceil q_c\kappa_1\rceil\leq \lceil q_c\kappa_2\rceil$. For the second constraint, notice that $\sum_{c\in [N]}\max\{|A\cap U_c|,\lceil p_c \kappa_1\rceil\}\leq\kappa_1$ is equivalent to that $\sum_{c\in [N]}\max\{|A\cap U_c|/\kappa_1,\frac{\lceil p_c \kappa_1\rceil}{\kappa_1} \}\leq1$. It then follows that $$\sum_{c\in [N]}\max\{|A\cap U_c|/\kappa_2,\frac{\lfloor p_c \kappa_2\rfloor}{\kappa_2} \}\leq \sum_{c\in [N]}\max\{|A\cap U_c|/\kappa_1,\frac{\lceil p_c \kappa_1\rceil}{\kappa_1}  \}\leq1.$$ Therefore, $A\in \mathcal{M}_{fair}(P,\kappa_2,\lfloor \vec{p}\kappa_2\rfloor,\lceil\vec{q}\kappa_2\rceil)$.
\end{proof}
We now prove Theorem~\ref{thm:convert}. 
\thmconvert*
\begin{proof}
    Denote the optimal solution of the FSC as $OPT$. Since by Lemma \ref{lem:fairness_for_cover}, the fairness constraint for cover is always satisfied. When the guess of 
$OPT$ satisfies that  $|OPT|<\kappa\leq(1+\alpha)|OPT|$, by the definition of bicriteria approximation algorithm for FSM, it follows that
    \begin{align*}
        f(S)\geq \gamma\max_{X\in\mathcal{M}_{fair}(P,\kappa,\lfloor \vec{p}\kappa\rfloor,\lceil\vec{q}\kappa\rceil)}f(X).
    \end{align*}
 Since $\kappa> |OPT|$, by Lemma \ref{lem:matroid_subset_relationship}, we have that $\mathcal{M}_{fair}(P,|OPT|,\lceil \vec{p}|OPT|\rceil,\lceil\vec{q}|OPT|\rceil)\subseteq\mathcal{M}_{fair}(P,\kappa,\lfloor \vec{p}\kappa\rfloor,\lceil\vec{q}\kappa\rceil)$. Therefore, it follows that

    \begin{align*}
        \max_{X\in\mathcal{M}_{fair}(P,\kappa,\lfloor \vec{p}\kappa\rfloor,\lceil\vec{q}\kappa\rceil)}f(X)\geq\max_{X\in\mathcal{M}_{fair}(P,|OPT|,\lceil \vec{p}|OPT|\rceil,\lceil\vec{q}|OPT|\rceil)}f(X)
    \end{align*}
    Since $OPT$ is the optimal solution of FSC, we have that $\lceil p_c|OPT|\rceil\leq |OPT\cap U_c |\leq \lceil q_c|OPT|\rceil$. It implies that $OPT\in\mathcal{M}_{fair}(P,|OPT|,\lceil \vec{p}|OPT|\rceil,\lceil\vec{q}|OPT|\rceil)$. Therefore we can get $$\max_{X\in\mathcal{M}_{fair}(P,|OPT|,\lceil \vec{p}|OPT|\rceil,\lceil\vec{q}|OPT|\rceil)}f(X)\geq f(OPT)\geq\tau.$$ Then 
    \begin{align*}
        f(S)\geq\gamma\tau.
    \end{align*}
    
    This means that the algorithm stops by the time when $\kappa$ reaches the region of $(|OPT|,(1+\alpha)|OPT|]$, which implies that the output solution set satisfies $|S|=\beta\kappa\leq\beta(1+\alpha)|OPT|$. Since by Lemma \ref{lem:fairness_for_cover}, the fairness constraint is always satisfied, the output solution set satisfies a $((1+\alpha)\beta,\gamma)$-approximation ratio. The algorithm makes $O(\log_{1+\alpha} |OPT|)$ calls to the bicriteria algorithm for FSM with $\kappa=1+\alpha,(1+\alpha)^2,...,(1+\alpha)|OPT|$, so the query complexity is $O(\sum_{i=1}^{\frac{\log(|OPT|)}{\log(\alpha+1)}}\mathcal{T}(n,(1+\alpha)^i))$.
\end{proof}

\subsection{Converting theorem for  continuous algorithms}
\label{appdx:conv_conti}
In this section, we present and analyze the converting algorithm for the continuous algorithms, which is denoted as \convc. The algorithm description is in Algorithm \ref{alg:conv_c}. The main result of the algorithm is presented in Theorem \ref{thm:conv_continuous}, which we restate as follows.

\textbf{Theorem \ref{thm:conv_continuous}. }\textit{
   Any continuous algorithm with a  $(\gamma,\beta)$-bicriteria approximation ratio for FSM that returns a solution in time $\mathcal{T}(n,\kappa)$ with probability at least $1-\frac{\delta}{n}$ can be converted into an approximation algorithm for FSC such that with probability $1-\delta$, the algorithm satisfies a $((1+\alpha)\beta,\frac{(1-\frac{\varepsilon}{2})\gamma-\frac{\varepsilon}{3}}{1+\frac{\varepsilon}{2}+\frac{\varepsilon}{3\gamma}})$-bicriteria approximation ratio where $\frac{(1-\frac{\varepsilon}{2})\gamma-\frac{\varepsilon}{3}}{1+\frac{\varepsilon}{2}+\frac{\varepsilon}{3\gamma}} $ holds in expectation. The query complexity is at most $O(\sum_{i=1}^{\log_{1+\alpha}|OPT|}\mathcal{T}(n,(1+\alpha)^i)+\frac{n\log_{1+\alpha}|OPT|}{\varepsilon^2}\log\frac{n}{\delta})$.
}
\begin{proof}
Throughout the proof, we use $OPT$ to denote the optimal solution of the FSM. In addition, we denote the optimal solution of FSM under the total cardinality $\kappa$ as $OPT_\kappa$, i.e., $OPT_\kappa=\arg\max_{S\in\mathcal{M}_{fair}(P,\kappa,\vec{p}\kappa,\vec{q}\kappa)}f(S)$. First of all, notice that there are at most $\min\{n,\log_{1+\alpha}|OPT|+1\}$ number of guesses of $|OPT|$ before $\kappa$ reaches $|OPT|\leq\kappa\leq (1+\alpha)|OPT|$. By taking a union bound over all guess of $|OPT|$ we would obtain with probability at least $1-\frac{\delta}{2}$ and for each guess of $|OPT|$, the algorithm for FSM outputs a solution $\vect{x}$ with a bicriteria approximation ratio of $(\gamma,\beta)$. 

Since $\vect{F}(\vect{x})\leq n\max_{s\in\mathcal{M}_\beta}f(s)\leq n f(OPT_\kappa)$, by the Chernoff bound in Lemma \ref{lem:chernoff} and taking the union bound, it follows that with probability at least $1-\frac{\delta}{2}$, for each guess of $|OPT|$, the estimate of $\vect{F}(\vect{x})$ in Line \ref{line:estimateY} of Algorithm \ref{alg:conv_c} denoted as $Y$, satisfies that
\begin{align*}
    |Y-\vect{F}(\vect{x})|\leq \frac{\varepsilon}{3} f(OPT_\kappa)+\frac{\varepsilon}{2}\vect{F}(\vect{x}).
\end{align*}
By the definition of the bicriteria approximation ratio, it follows that 
$Y\geq\{(1-\frac{\varepsilon}{2})\gamma-\frac{\varepsilon}{3}\} f(OPT_\kappa)$.

Similar to the proof of Theorem \ref{alg:conver}, we can see that when $\kappa$, which is the guess of the size $OPT$ satisfies that $|OPT|\leq\kappa\leq (1+\alpha)|OPT|$, it follows that 
\begin{align*}
    f(OPT_\kappa)\geq\tau.
\end{align*}
Therefore, $Y\geq \{(1-\frac{\varepsilon}{2})\gamma-\frac{\varepsilon}{3}\} \tau$. It then follows that the algorithm stops before the guess of $|OPT|$ satisfies $|OPT|\leq\kappa\leq (1+\alpha)|OPT|$. The value of multi-linear extension of the output fractional solution then satisfies 
\begin{align*}
    (1+\frac{\varepsilon}{2})\vect{F}(\vect{x})+\frac{\varepsilon}{3} f(OPT_\kappa)\geq Y\geq\{(1-\frac{\varepsilon}{2})\gamma-\frac{\varepsilon}{3}\} \tau.
\end{align*}
Combining the above inequality with that $\vect{F}(\vect{x})\geq\gamma f(OPT_\kappa)$ Then
\begin{align*}
    \vect{F}(\vect{x})\geq\frac{(1-\frac{\varepsilon}{2})\gamma-\frac{\varepsilon}{3}}{1+\frac{\varepsilon}{2}+\frac{\varepsilon}{3\gamma}} \tau.
\end{align*}
Since $\vect{x}\in\mathcal{P}(\mathcal{M}_\beta)$, where $\mathcal{M}_\beta$ is the $\beta$ extension of the fairness matroid under the guess $\kappa$, then after the pipage rounding step, we would have that $S\in\mathcal{M}_\beta$, and the value of objective function satisfies $\mathbb{E} f(S)\geq\vect{F}(\vect{x})\geq\frac{(1-\frac{\varepsilon}{2})\gamma-\frac{\varepsilon}{3}}{1+\frac{\varepsilon}{2}+\frac{\varepsilon}{3\gamma}} \tau$. After the rounding steps from Line \ref{line:convc,rounding_starts} to Line \ref{line:convc_rounding2_ends} in Algorithm \ref{alg:conv_c}, we would get that the final solution set satisfies 
\begin{align}
    &\mathbb{E}f(S)\geq\frac{(1-\frac{\varepsilon}{2})\gamma-\frac{\varepsilon}{3}}{1+\frac{\varepsilon}{2}+\frac{\varepsilon}{3\gamma}} \tau\nonumber\\
    &\beta\lfloor \frac{p_c|S|}{\beta}\rfloor\leq |S\cap U_c|\leq \beta\lceil \frac{q_c|S|}{\beta}\rceil\nonumber\\
    &|S|\leq (1+\alpha)\beta|OPT|\nonumber.
\end{align}

\end{proof}
\begin{algorithm}[t!]
\caption{\convc}\label{alg:conv_c}
\textbf{Input}: An FSC instance with threshold $\tau$, fairness parameters $\vec{p}$, $\vec{q}$, partition $P$, a $(\gamma,\beta)$-bicriteria approximation algorithm for FSM, $\alpha>0$\\
\textbf{Output}: $S\subseteq U$
\begin{algorithmic}[1]
\State $\kappa\gets\lceil1+\alpha\rceil$, $S\gets\emptyset$.
\While{true}
\State $\vect{x}\gets(\gamma,\beta)$-bicriteria approximation for FSM with fairness matroid 
$\mathcal{M}_{fair}(P,\kappa,\vec{p}\kappa,\vec{q}\kappa)$
\label{line:convcmax}
\State $Y\gets$ average over $\frac{18n}{\varepsilon^2}\log(\frac{4n}{\delta})$ samples from $\vect{F}(\vect{x})$\label{line:estimateY}
\If{$Y\geq\{(1-\frac{\varepsilon}{2})\gamma-\frac{\varepsilon}{3}\} \tau$}
\State $S\gets$ pipage rounding of $\vect{x}$
 \For {$c\in[N]$} \label{line:convc,rounding_starts}
        \If{$|S\cap U_c| < p_c\beta\kappa$} 
        \State Add new elements from $U_c/S$ to $S$ 
        until $|S\cap U_c| \geq p_c\beta\kappa$
        \EndIf
\EndFor\label{line:convc,rounding_ends}
        \If{$|S| < \beta\kappa$} 
        \For {$c\in[N]$} \label{line:convc_rounding2_starts}
        \While{$|S| < \beta\kappa$ and $|S\cap U_c|< q_c \beta \kappa$}
        \State Add new elements in $U_c/S$ to $S$
        \EndWhile
        \EndFor\label{line:convc_rounding2_ends}
        \EndIf
        \EndIf
\State $\kappa\gets\lceil(1+\alpha)\kappa\rceil$
\EndWhile \label{line:convc_alg_convert,loop_ends}
\State \textbf{return} $S$
\end{algorithmic}
\end{algorithm}

\section{Appendix for Section \ref{sec:alg_for_FSM}}
\label{appdx:alg_for_FSM}
In this section, we present the missing content in Section \ref{sec:alg_for_FSM} in the main paper. 

\subsection{Appendix for Section \ref{sec:discrete}}
In this portion of appendix, we present the missing details and proofs in Section \ref{sec:discrete} in the main paper, which is about two discrete algorithms \greedyalg and \threalglong. We begin by presenting the proof of Lemma \ref{lem:feasible_OPT}, followed by proofs of the threshold greedy algorithm \threalglong. Finally, the pseudocode of \greedyalg and \threalglong are presented in Algorithm \ref{alg:fairness-bi} and Algorithm \ref{alg:thres-fairness-bi} respectively.

First of all, we prove Lemma \ref{lem:feasible_OPT}, which builds the relationship between the original fairness matroid and its $\beta$-extension for any $\beta\in \mathbb{N}_+$. 

\noindent\textbf{ Lemma
   \ref{lem:feasible_OPT}. }\textit{For any $\beta\in\mathbb{N}_+$ and any fairness matroid $\mathcal{M}_{fair}(P,\kappa,\vec{l},\vec{u})$, denote $\mathcal{M}_{\beta}$ as the $\beta$-extended fairness matroid of $\mathcal{M}_{fair}(P,\kappa,\vec{l},\vec{u})$. Then for any set $S\in\mathcal{M}_{\beta}$ with $|S|=\beta \kappa$, $T\in \mathcal{M}_{fair}(P,\kappa,\vec{l},\vec{u})$ with $|T|=\kappa$, and any permutation of $S=(s_1,s_2,...,s_{\beta\kappa})$, there exist a sequence $E=(e_1,e_2,...,e_{\beta\kappa})$ such that each element in $T$ appears $\beta$ times in $E$ and that 
   \begin{align*}
        S_i\cup \{e_{i+1}\}\in \mathcal{M}_\beta , \qquad\forall i\in\{0,1,...,\beta\kappa\}
    \end{align*}
    where $S_i=(s_1,s_2,...,s_i)$ and $S_0=\emptyset$.
}
\begin{proof}
    Before proving the lemma, we define some notations here. For any sequence of any length $m$ denoted as $A=(a_1,a_2,...,a_m)$, we define the number of element $x$ in the sequence as $|A^x|$, i.e., $|A^x|:=|\{i:a_i=x\}|$. In addition, we define the number of elements of group $c$ in the sequence as $|A^c|$, i.e., $|A^c|=|\{i:a_i\in U_c\}|=\sum_{x\in U_c}|A^x|$. For the sequence $E$, we denote the sequence containing $i$-th element to the last elements as $E_i$, i.e., $E_i=(e_{i}, e_{i+1},...,e_{\beta\kappa})$. Now we prove a stronger claim which would imply the results in the lemma.
    \begin{claim}
        For any $\beta\in\mathbb{N}_+$, denote $\mathcal{M}_{\beta}$ as the $\beta$-extension of the fairness matroid. Then for any set $S\in\mathcal{M}_{\beta}$, there exists a sequence $E=(e_1,...,e_{\beta \kappa})$ such that for each $i\in\{0,1,...,\beta\kappa\}$, the sequence $F_i=(S_{i},E_{i+1})=(s_1,s_2,...,s_{i},e_{i+1}, ...,e_{\beta \kappa})$ satisfies that 
\begin{align*}
    &|F_i^c|\leq u_c \beta\qquad \forall c\in[N]\\
    &\sum_{c\in[N]}\max\{|F_i^c|,l_c\beta\}\leq \beta \kappa.  
\end{align*}
Here for each $e\in E$, we have $e\in T$.  Besides, we have that for any element $x\in T$,
\begin{align*}
    |F_i^x|\leq \beta.
\end{align*}
\end{claim}
    We prove the claim by induction. First, when $i=\beta\kappa$, $F_i=S$. Since $S\in\mathcal{M}_\beta$, the claim holds. Suppose the result in the claim holds for $i$, and we prove the claim for $i-1$. There are two cases.
    \begin{itemize}
        \item Case 1. There exists some group $c_0$ such that $|(S_{i-1},E_{i+1})^{c_0}|\leq l_{c_0} \beta-1$. Since $|T|=\kappa$, $|T\cap U_c|\geq l_c$ for each $c\in[N]$. Therefore, in this case, there exists at least one element $x\in U_{c_0}\cap T$ such that $|(S_{i-1},E_{i+1})^{x}|< \beta$. Then choose $e_i=x$ and $E_i=(x, E_{i+1})$, the results in the claim will be satisfied.
        \item Case 2. For all group $c\in[N]$, $|(S_{i-1},E_{i+1})^{c}|\geq l_c \beta$. Since the sequence $(S_{i-1},E_{i+1})$ is of length $\beta \kappa-1$, we have that $$|(S_{i-1},E_{i+1})|<\beta \kappa\leq|T| \beta.$$ Therefore, there exists at least one group $c_1$ such that $|(S_{i-1},E_{i+1})^{c_1}|< |T\cap U_{c_1}| \beta$. (Otherwise $\sum_{c\in[N]}|(S_{i-1},E_{i+1})^{c}|\geq\sum_{c\in[N]}|T\cap U_c| \beta=\beta \kappa$, which breaks the assumption.) 
        From $|(S_{i-1},E_{i+1})^{c_1}|< |T\cap U_{c_1}| \beta$, we have that there exists at least one element $x\in T\cap U_{c_1}$ such that $$|(S_{i-1},E_{i+1})^{x}|\leq \beta-1.$$ Then we set the $i$-th element in $E$ to be $x$, then $E_{i}=(x,E_{i+1})$. It follows that $|(S_{i-1},E_{i})^{x}|\leq \beta$. For each element $x'\in T/\{x\}$ , $|(S_{i-1},E_{i})^{x'}|=|(S_{i-1},E_{i+1})^{x'}|\leq \beta$. 
        Since $e_i=x\in U_{c_1}$. For group $c\neq c_1$, $|(S_{i-1},E_{i})^{c}|=|(S_{i-1},E_{i+1})^{c}|\leq u_c\beta$ by the assumption that the claim holds for iteration $i$. For group $c_{1}$, $|(S_{i-1},E_{i})^{c_1}|=|(S_{i-1},E_{i})^{c}|+1\leq u_c\beta$. Since for all group  $c\in[N]$, $|(S_{i-1},E_{i+1})^{c}|\geq l_c\beta$, it follows that $|(S_{i-1},E_{i})^{c}|\geq l_c\beta$. Thus
        \begin{align*}
        \sum_{c\in[N]}\max\{|(S_{i-1},E_{i})^c|,l_c\beta\}&=\sum_{c\in[N]}|(S_{i-1},E_{i})^c|\\
        &=|(S_{i-1},E_{i})|= \beta\kappa.
        \end{align*} 
        Thus we prove the claim for iteration $i-1$ under the assumption that the claim holds for $i$. By induction, the claim holds for all $i$. For $i=0$, $(S_{0},E_{0})=E$. From the construction of $E$ we have that  $|E|=\beta\kappa$, and that $|E^o|=\beta$ for all $o\in T$. Since for each group $c$, we have $|S_{i}\cup \{e_{i+1}\}\cap U_c|\leq |(S_{i},E_{i})^c|$. From the result in the claim, we can prove that $S_i\cup \{e_{i+1}\}\in\mathcal{M}_{\beta}$.
    \end{itemize}

\end{proof}

\subsubsection{Proof of Theorem~\ref{thm:greedy}}
\thmgreedy*
\begin{proof}
    Denote the optimal solution of $\max_{S\in\mathcal{M}_{fair}(P,\kappa,\vec{l},\vec{u})}f(S)$ as $OPT$, i.e., $OPT=\arg\max_{S\in\mathcal{M}_{fair}(P,\kappa,\vec{l},\vec{u})}f(S)$.  Since by Lemma \ref{lem:assump_for_SM}, we have $\mathcal{M}_{1/\varepsilon}=\mathcal{M}_{fair}(P,\kappa /\varepsilon,\vec{l}/\varepsilon,\vec{u}/\varepsilon)$ is a matroid of rank $\kappa /\varepsilon$, then the Algorithm \ref{alg:fairness-bi} ends after $\kappa /\varepsilon$ steps and the output solution set satisfies $|S|=\kappa /\varepsilon$. Since $S\in\mathcal{M}_{1/\varepsilon}$, then 
    \begin{align*}
        &|S\cap U_c|\leq u_c/\varepsilon \qquad\forall c\in[N]\\
        &\max_{c\in[N]}\{|S\cap U_c|,l_c/\varepsilon\}\leq \kappa /\varepsilon.
    \end{align*}

    Then it remains to prove that $f(S)\geq (1-\varepsilon)f(OPT)$. From Lemma \ref{lem:feasible_OPT}, we know that there exists a sequence $E$ that contains $1/\varepsilon$ copies of $OPT$ and that at each step $i$, $S_i\cup\{e_{i+1}\}\in\mathcal{M}_{1/\varepsilon}$. Then by the greedy selection strategy, we have 
    \begin{align*}
        f(S_{i+1})-f(S_i)\geq f(S_i\cup\{e_{i+1}\})-f(S_i).
    \end{align*}
    Thus by submodularity, we have 
    \begin{align*}
        f(S_{i+1})-f(S_i)\geq f(S_i\cup\{e_{i+1}\})-f(S_i)\geq\Delta f(S,e_{i+1}).
    \end{align*}
    Summing over all $i$, we would get
    \begin{align*}
        \sum_{i=0}^{\frac{k}{\varepsilon}-1}f(S_{i+1})-f(S_i)\geq\sum_{i=0}^{\frac{k}{\varepsilon}-1}\Delta  f(S,e_{i+1}).
    \end{align*}
    Since the sequence $E$ contains $1/\varepsilon$ copies of each element in $OPT$, then $\sum_{i=0}^{\frac{k}{\varepsilon}-1} \Delta f(S,e_{i+1})=1/\varepsilon\sum_{o\in OPT}\Delta f(S,o)$. Since $ \sum_{i=0}^{\frac{k}{\varepsilon}-1}f(S_{i+1})-f(S_i)=f(S)-f(\emptyset)$ and that $f$ is nonnegative,
    \begin{align*}
        f(S)\geq\sum_{i=0}^{\frac{k}{\varepsilon}-1}\Delta  f(S,e_{i+1})\geq1/\varepsilon\sum_{o\in OPT}\Delta f(S,o)\geq\frac{f(OPT)-f(S)}{\varepsilon }.
    \end{align*}
   Thus we have
    \begin{align*}
        f(S) \geq \frac{1}{1 + \varepsilon} f(OPT) \geq (1 - \varepsilon) f(OPT).
    \end{align*}
\end{proof}

\subsubsection{Proof of Theorem \ref{theorem:threshold-fairness-bi}}
Before we present the proof of the theorem, first we present the proof of the following lemma. Let us denote the solution set after the $i$-th element in \texttt{threshold-fairness-bi} as $S_i$. By Lemma \ref{lem:feasible_OPT}, we know that we can construct a sequence $E=(e_1,e_2,..,e_{\kappa/\varepsilon})$ that contains $1/\varepsilon$ copies of $OPT$ and that $S_i\cup \{e_{i+1}\}\in\mathcal{M}_{1/\varepsilon}$. Then we have the following lemma.
\begin{lemma}
\label{lem:iterative_thres}
    For any $0\leq i<\kappa/\varepsilon$, it follows that
    \begin{align*}
        \Delta f(S_{i}, s_{i+1}) \geq (1-\varepsilon) \Delta f(S_{i}, e_{i+1})-\varepsilon d/\kappa.
    \end{align*}
\end{lemma}
\begin{proof}
        First, we consider the case if $s_{i+1}$ is added to the solution set and is not a dummy variable, it follows that $\Delta f(S_{i}, s_{i+1}) \geq \tau$. Since $S_i\cup\{e_{i+1}\}\in\mathcal{M}_{1/\varepsilon}$, then if $e_{i+1}\notin S_i$, by submodularity we have $\Delta f(S_{i}, e_{i+1})\leq\tau/(1-\varepsilon) $. If $e_{i+1}\in S_i$, then $\Delta f(S_{i}, e_{i+1})=0\leq\tau/(1-\varepsilon) $. 
        Next, we consider the case if $s_{i+1}$ is a dummy variable, then $\Delta f(S_i,s_{i+1})=0$. If $e_{i+1}\in S_i$, then $\Delta f(S_{i}, e_{i+1})=0$ and the above inequality in the lemma holds. If $e_{i+1}\notin S_i$, since $S_{\kappa_1}\cup\{e_{i+1}\}\in\mathcal{M}_{1/\varepsilon}$, then $\Delta f(S_{i}, e_{i+1})\leq\varepsilon d/\kappa$. Therefore, we have that
   \begin{align*}
        \Delta f(S_{i}, s_{i+1}) \geq (1-\varepsilon) \Delta f(S_{i}, e_{i+1})-\varepsilon d/\kappa.
    \end{align*}
\end{proof}

Next, we present the proof of Theorem \ref{theorem:threshold-fairness-bi}. 
\thmthresholdfairnessbi*
\begin{proof}
   First, notice that the Algorithm \ref{alg:thres-fairness-bi} ends in at most $(1/ \varepsilon) \log(\kappa /\varepsilon)$ number of iterations. Therefore, there are at most $n/\varepsilon\log(\kappa/\varepsilon)$ number of queries to $f$. Next, we prove the bicriteria approximation ratio of \texttt{threshold-fairness-bi}. From the description of Algorithm \ref{alg:thres-fairness-bi}, we have that the output solution set $S\in\mathcal{M}_{1/\varepsilon}$, then 
   \begin{align*}
        &|S\cap U_c|\leq u_c/\varepsilon \qquad\forall c\in[N]\\
        &\max_{c\in[N]}\{|S\cap U_c|,l_c/\varepsilon\}\leq \kappa /\varepsilon.
    \end{align*}
    It remains to prove that $f(S)\geq(1-2\varepsilon)f(OPT)$ where $OPT$ is defined as the optimal solution of FSM, i.e., $OPT=\arg\max_{S\in\mathcal{M}_{fair}(P,\kappa,\vec{l},\vec{u})}f(S)$.  For simplicity, we assume the returned solution has size $|S| = k_1$. As discussed in the proof of Theorem \ref{thm:greedy}, Rank$(\mathcal{M}_{1/\varepsilon}) = \kappa /\varepsilon$. Here we denote the solution set as $S=(s_1,s_2,...,s_{\kappa/\varepsilon})$, and we define $S_i$ as $S_i=(s_1,...,s_i)$. Here $s_i$ is the $i$-th element added to the solution set. 
   In the case when the threshold $\tau$ drops below $\varepsilon d / \kappa$ at the termination and $\kappa_1\leq\kappa/\varepsilon$, we can add dummy elements to $S$ such that $|S| = \kappa / \varepsilon$. By Lemma \ref{lem:iterative_thres}, we have that there is a sequence $E=(e_1,e_2,...,e_{k/\varepsilon})$ that contains $1/\varepsilon$ copies of $OPT$ and 
    
    \begin{align*}
        \Delta f(S_{i}, s_{i+1}) \geq (1-\varepsilon) \Delta f(S_{i}, e_{i+1})-\varepsilon d/\kappa.
    \end{align*}
    Thus by submodularity, we have 
    \begin{align*}
        f(S_{i+1})-f(S_i)\geq (1-\varepsilon) \{f(S_i\cup\{e_{i+1}\})-f(S_i)\}-\varepsilon d/\kappa \geq (1-\varepsilon) \Delta f(S,e_{i+1})-\varepsilon d/\kappa.
    \end{align*}
    Summing over all $i$, we would get
    \begin{align*}
        \sum_{i=0}^{k/\varepsilon-1}f(S_{i+1})-f(S_i) \geq (1-\varepsilon) \sum_{i=0}^{\kappa/\varepsilon-1}\Delta  f(S,e_{i+1})-d.
    \end{align*}
    Since the sequence $E$ contains $1/\varepsilon$ copies of each element in $OPT$, then $\sum_{i=0}^{\frac{k}{\varepsilon}-1} \Delta f(S,e_{i+1})=1/\varepsilon\sum_{o\in OPT}\Delta f(S,o)$. Since $ \sum_{i=0}^{\kappa/\varepsilon-1}f(S_{i+1})-f(S_i)=f(S)-f(\emptyset)$ and that $f$ is nonnegative,
    \begin{align*}
        f(S) &\geq (1-\varepsilon) \sum_{i=0}^{\kappa/\varepsilon-1} \Delta f(S,e_{i+1})-d\\
        &\geq \frac{(1-\varepsilon)}{\varepsilon} \sum_{o\in OPT} \Delta f(S,o) - d\\
        &\geq \frac{1}{\varepsilon} \{(1-\varepsilon)\{f(OPT)-f(S)\} - \varepsilon f(OPT)\}\\
        &\geq \frac{1}{\varepsilon} \{(1-2\varepsilon)f(OPT)-(1-\varepsilon)f(S)\} .
    \end{align*}
   By re-arranging the above equation, we have that
    \begin{align*}
        f(S) \geq (1 - 2\varepsilon)  f(OPT).
    \end{align*}
\end{proof}

\begin{algorithm}[t]
    \caption{\texttt{greedy-fairness-bi}}\label{alg:fairness-bi}
    \begin{algorithmic}[1]
        \State \textbf{Input: }$\varepsilon$, fairness matroid $\mathcal{M}_{fair}(P,\kappa,\vec{l},\vec{u})$
        \State \textbf{Output: }$S\in U$
        \State $S\gets\emptyset$
        \State Denote $\mathcal{M}_{fair}(P,\kappa /\varepsilon,\vec{l}/\varepsilon,\vec{u}/\varepsilon)$ as $\mathcal{M}_{1/\varepsilon}$.
        \While{$\exists x$ s.t. $S\cup\{x\}\in\mathcal{M}_{1/\varepsilon}$}
        \State $V \gets\{x\in U|S \cup \{x\} \in \mathcal{M}_{1/\varepsilon}\} $
        \State $u \gets\arg\max_{x\in V} \Delta f(S,x)$
        \State $S \gets S \cup \{u\}$
        \EndWhile
        \textbf{return }$S$
    \end{algorithmic}
\end{algorithm}

\begin{algorithm}[t]
    \caption{\texttt{threshold-fairness-bi}}\label{alg:thres-fairness-bi}
    \begin{algorithmic}[1]
        \State \textbf{Input: }$\varepsilon$, fairness matroid $\mathcal{M}_{fair}(P,\kappa,\vec{l},\vec{u})$
        \State \textbf{Output: }$S\in U$
        \State $S\gets\emptyset$
        \State Denote $\mathcal{M}_{fair}(P,\kappa /\varepsilon,\vec{l}/\varepsilon,\vec{u}/\varepsilon)$ as $\mathcal{M}_{1/\varepsilon}$
        \State $d \gets \max_{\{x\} \in M_{1/\varepsilon}} f(\{x\})$
        \For{$\tau=d; \tau \geq  \varepsilon d/k; \tau\gets \tau(1-\varepsilon)$}\label{line:for_loop_starts_thres}
        \For{$x \in U$}
        \If{$S \cup \{x\} \in \mathcal{M}_{1/\varepsilon}$ and $\Delta f(S, x) \geq \tau$}
        \State $S\gets S\cup \{x\}$
        \EndIf
        \If{$|S|= \kappa /\varepsilon$}
        \State \textbf{return }$S$
        \EndIf
        \EndFor
        \EndFor\label{line:for_loop_ends_thres}
        \textbf{return }$S$
    \end{algorithmic}
\end{algorithm}

\subsection{Appendix for Section \ref{sec:continuous}}
\label{appdx:continuous}
In this section, we provide the omitted content from Section \ref{sec:continuous} of the main paper. Specifically, we present the proof of Lemma \ref{lem:Continuous_decreasing_threshold}, which offers the theoretical guarantee for the subroutine algorithm \contisublong of the continuous algorithm \contialg. The statement of the Lemma is as follows.
\begin{lemma}
     \label{lem:Continuous_decreasing_threshold}
     During each call of \contisublong, the output coordinate set $B$ satisfies that
     \begin{align*}
         \vect{F}(\vect{x}+\varepsilon\vect{1}_B)-\vect{F}(\vect{x})&\geq\varepsilon\{\ln(1/\varepsilon)+1\}((1-6\varepsilon)f(OPT)-\vect{F}(\vect{x}+\varepsilon\vect{1}_B)).
     \end{align*}
     
 \end{lemma}

 \begin{proof}
     For notation simplicity, we denote the rank of the matroid $\mathcal{M}_{\ln(\frac{1}{\varepsilon})+1}$ as $m$, i.e., $m:=(\ln(\frac{1}{\varepsilon})+1)\kappa$. Here we denote the output solution set as  $B=\{b_1,b_2,...,b_{m}\}$ where $b_i$ is the $i$-th element that is added to set $B$. Here if $|B|<m$, then for any $i>|B|$, $b_i$ is defined as a dummy variable. In addition, we define $B_i=\{b_1,b_2,...,b_i\}$. Since $\mathcal{M}_{\ln(\frac{1}{\varepsilon})+1}$ is an $\ln(\frac{1}{\varepsilon})+1$-extension of the original fairness matroid, then by Lemma \ref{lem:feasible_OPT}, there exists a sequence $E=(e_1,e_2,...,e_m)$ such that $E$ contains $\ln(1/\varepsilon)+1$ copies of the optimal solution $OPT=\{o_1,o_2,...,o_\kappa\}$ such that $B_{i-1}\cup \{e_i\}\in\mathcal{M}_{\ln(1/\varepsilon)+1}$ for each $i\in[m]$. 

     Notice that by Lemma \ref{lem:chernoff}, we have that with probability at least $1-\frac{\varepsilon^3}{2n^4}$, for any fixed $\vect{x}+\vect{1}_B$ and fixed element $u$, the empirical mean $\hat{X}(B_i,u)$, which is the average over $\frac{3\kappa}{\varepsilon^2}\log{\frac{4n^4}{\varepsilon^3}}$ samples of the random variable $X={\Delta f}(S(\vect{x}+\varepsilon\vect{1}_{B_i}),u)$ satisfies that
     \begin{align}
     \label{ineq:chernoff}
         |\hat{X}(B_i,u)-\mE \Delta f(S(\vect{x}+\varepsilon\vect{1}_{B_i}),u)|\leq \frac{\varepsilon}{\kappa}f(OPT)+\varepsilon\mE \Delta f(S(\vect{x}+\varepsilon\vect{1}_{B_i}),u).
     \end{align}
     Since during the execution of \contialg, there are at most $\frac{n}{\varepsilon^2}\log(\kappa/\varepsilon)$ such estimations, by applying the union bound, we have that with probability at least $1-\frac{1}{2n^2}$, the inequality (\ref{ineq:chernoff}) holds for each $\vect{x}$, $B$ and 
  $u$ during \contialg.
 From the description in Algorithm \ref{alg:CCTG_subroutine}, we can see that $\hat{X}(B_{i-1},b_i)\geq w$. For the element $e_i$, we have that $\hat{X}(B_{i-1},e_i)\leq \frac{w}{1-\varepsilon}$ or at the last iteration, we have that $\hat{X}(B_{i-1},e_i)\leq \frac{\varepsilon d}{\kappa}$. Therefore, we have that $\hat{X}(B_{i-1},b_i)\geq (1-\varepsilon)\hat{X}(B_{i-1},e_i)-\frac{\varepsilon d}{\kappa}$. Since $OPT=\max_{S\in\mathcal{M}_{fair}(P,\kappa,\vec{p}\kappa,\vec{q}\kappa)}f(S)\geq d$, it then follows that
\begin{align*}
    (1+\varepsilon)\mE \Delta f(S(\vect{x}+\varepsilon\vect{1}_{B_{i-1}}),b_i)\geq(1-\varepsilon)^2\mE \Delta f(S(\vect{x}+\varepsilon\vect{1}_{B_{i-1}}),e_i)-\frac{3\varepsilon}{\kappa}f(OPT).
\end{align*}
By rearranging the above inequality and simple calculations, we have 

     \begin{align*}
    \mE \Delta f(S(\vect{x}+\varepsilon\vect{1}_{B_{i-1}}),b_i)\geq(1-3\varepsilon)\mE \Delta f(S(\vect{x}+\varepsilon\vect{1}_{B_{i-1}}),e_i)-\frac{3\varepsilon}{\kappa}f(OPT).
\end{align*}
     By the construction of set $B$, we would get
     \begin{align*}
    \vect{F}(\vect{x}+\varepsilon\vect{1}_B)-\vect{F}(\vect{x})&=\sum_{i=1}^m\vect{F}(\vect{x}+\varepsilon\vect{1}_{B_i})-\vect{F}(\vect{x}+\varepsilon\vect{1}_{B_{i-1}})\\
    &=\sum_{i=1}^m \varepsilon\cdot\frac{\partial \vect{F}}{\partial b_i}\big|_{x=\vect{x}+\vect{1}_{B_{i-1}}}\\
    &\geq \varepsilon\sum_{i=1}^m\mE \Delta f(S(\vect{x}+\varepsilon\vect{1}_{B_{i-1}}),b_i)\\
    &\geq\varepsilon\sum_{i=1}^m (1-3\varepsilon)\mE \Delta f(S(\vect{x}+\varepsilon\vect{1}_{B_{i-1}}),e_i)-\frac{3\varepsilon f(OPT)}{\kappa}\\
    &=\varepsilon(1-3\varepsilon)\sum_{i=1}^m \mE \Delta f(S(\vect{x}+\varepsilon\vect{1}_{B_{i-1}}),e_i)\\
    &\qquad-3\varepsilon^2(\ln(\frac{1}{\varepsilon})+1) f(OPT)\\
    &\geq\varepsilon(1-3\varepsilon)\sum_{i=1}^m \mE \Delta f(S(\vect{x}+\varepsilon\vect{1}_{B}),e_i)\\
    &\qquad-3\varepsilon^2(\ln(\frac{1}{\varepsilon})+1) f(OPT),
\end{align*}
where the last inequality results from the submodularity of $f$. From Lemma \ref{lem:feasible_OPT}, we have that
\begin{align*}
    \vect{F}(\vect{x}+\varepsilon\vect{1}_B)-\vect{F}(\vect{x})
    &\geq\varepsilon(1-3\varepsilon)(\ln(\frac{1}{\varepsilon})+1)\sum_{i=1}^\kappa \mE \Delta f(S(\vect{x}+\varepsilon\vect{1}_{B}),o_i)\\
    &\qquad-3\varepsilon^2(\ln(\frac{1}{\varepsilon})+1) f(OPT)\\
    &\geq\varepsilon(1-3\varepsilon)(\ln(\frac{1}{\varepsilon})+1)\{f(OPT)-\vect{F}(\vect{x}+\varepsilon\vect{1}_B)\}\\
    &\qquad-3\varepsilon^2(\ln(\frac{1}{\varepsilon})+1) f(OPT)\\
    &\geq\varepsilon\{\ln(1/\varepsilon)+1\}((1-6\varepsilon)f(OPT)-\vect{F}(\vect{x}+\varepsilon\vect{1}_B)).
\end{align*}
 \end{proof}

\subsubsection{Proof of Theorem~\ref{thm:continuous}}
\thmcontinuous*
\begin{proof}
First of all, from the description of the subroutine algorithm \contisub in Algorithm \ref{alg:CCTG_subroutine}, we can see that there are at most $\log(\kappa/\varepsilon)/\varepsilon$ number of iterations in the outer for loop. Therefore, the subroutine algorithm \contisub takes at most $O(\frac{n\kappa\ln (n/\varepsilon)\ln(\kappa/\varepsilon)}{\varepsilon^3})$. Since there are at most $\frac{1}{\varepsilon}$ calls to \contisub, we can prove the sample complexity. 

Next, we prove the bicriteria approximation ratio. By Definition \ref{def:bicri-SM_conti}, it is equivalent to prove that $\vect{x}\in \mathcal{P}(\mathcal{M}_{\ln(1/\varepsilon)+1}) $ and
 $\vect{F}(\vect{x})\geq (1-7\varepsilon)f(OPT)$. Denote $B^{(t)}$ to be the output set of the $t$-th call to the subroutine algorithm \contisub.  Then it follows that the output solution set $\vect{x}$ of \contialg can be denoted as $\vect{x}=\sum_{t=1}^{1/\varepsilon}\varepsilon\vect{1}_{B^{(t)}}$. Since $B^{(t)}\in\mathcal{M}_{\ln(1/\varepsilon)+1}$, we have that $\vect{1}_{B^{(t)}}\in \mathcal{P}(\mathcal{M}_{\ln(1/\varepsilon)+1})$. 
By the fact that $\mathcal{P}(\mathcal{M}_{\ln(1/\varepsilon)+1})$ is convex, we have that $\vect{x}\in\mathcal{P}(\mathcal{M}_{\ln(1/\varepsilon)+1})$. Denote the fractional solution $\vect{x}$ after $t$-th step as $\vect{x}_t$, then by Lemma \ref{lem:Continuous_decreasing_threshold}, we have
    \begin{align*}
        \vect{F}(\vect{x}_{t+1})-\vect{F}(\vect{x}_{t})&\geq\varepsilon\{\ln(1/\varepsilon)+1\}((1-6\varepsilon)f(OPT)-\vect{F}(\vect{x}_{t+1})).
    \end{align*} 
    For notation simplicity, we define $L=\varepsilon\{\ln(1/\varepsilon)+1\}$. It then follows that
    \begin{align*}
        \vect{F}(\vect{x}_{t+1})&\geq\frac{\vect{F}(\vect{x}_{t})+L(1-6\varepsilon)f(OPT)}{1+L}
    \end{align*}
     Since there are $1/\varepsilon$ iterations in \contialg, the output $\vect{x}$ satisfies that $\vect{x}=\vect{x}_{1/\varepsilon}$. By applying induction to the above inequality, we would get
    \begin{align*}
        \vect{F}(\vect{x}_{1/\varepsilon})
        &\geq(1-(1+L)^{-1/\varepsilon})\{(1-6\varepsilon)f(OPT)\}\\
        &\geq(1-e^{\frac{1}{\varepsilon}(\frac{L^2}{2}-L)})\{(1-6\varepsilon)f(OPT)\}\\
        &\geq(1-\varepsilon)(1-6\varepsilon)f(OPT)\\
        &\geq (1-7\varepsilon)f(OPT).
    \end{align*}
\end{proof}
Observe that compared to \greedyalg and \threalglong, our proposed algorithm \contialg demonstrates an enhanced approximation ratio for the cardinality of the output solution set, improving from $O(1/\varepsilon)$ to $O(\ln(1/\varepsilon))$. This improvement is achieved while maintaining the same order of function value violation, specifically $f(S) \geq (1 - O(\varepsilon)) f(OPT)$. However, this enhancement requires an increased number of queries.  This suggests the potential for attaining comparable approximation ratios for the submodular cover problem under specific types of matroid-type constraints. 
\begin{corollary}\label{coro:combine_convert_with_conti}
    Using the Algorithm \ref{alg:CCTG_subroutine} as the subroutine for the converting algorithm in Algorithm \ref{alg:conver}, we obtain an algorithm that achieves an approximation ratio of $((1+\alpha)\ln(\frac{1}{\varepsilon})+1,1-10\varepsilon)$ in at most $O\left(\frac{n|OPT|\log^2\frac{n}{\varepsilon}}{\varepsilon^4\alpha}\right)$ with high probability.
\end{corollary}
\begin{proof}
    The result of the approximation ratio can be obtained by combining Theorem \ref{thm:conv_continuous} and Theorem \ref{thm:continuous} together. Here we provide proof of the sample complexity. Notice that for each guesses of OPT of cardinality $\kappa_g$, the algorithm \contialg that runs with the input $\mathcal{M}_{fair}(P,\kappa_g,\vec{p}\kappa_g,\vec{u}\kappa_g)$ uses at most $O(\frac{n\kappa_g\log^2(\frac{n}{\varepsilon})}{\varepsilon^4})$. From the result in Theorem \ref{thm:conv_continuous} and Theorem \ref{thm:continuous}, the total number of sample complexity would be 
    \begin{align*}
        \sum_{i=1}^{\frac{\log(|OPT|)}{\log(\alpha+1)}}O\left(\frac{n(1+\alpha)^{i}\log^2(\frac{n}{\varepsilon})}{\varepsilon^4}\right)+O\left(\frac{n\log_{1+\alpha}|OPT|}{\epsilon^2}\log n\right)=O\left(\frac{n|OPT|\log^2\frac{n}{\varepsilon}}{\varepsilon^4\alpha}\right).
    \end{align*}
\end{proof}

\section{Proof of Technical Lemmas}


 \begin{lemma}[Relative $+$ Additive Chernoff Bound (Lemma 2.3 in \cite{badanidiyuru2014fast})]
\label{lem:chernoff}
    Let $X_1,...,X_N$ be independent random variables such that for each $i$, $X_i\in[0,R]$ and $\mathbb{E}[X_i]=\mu$ for all $i$. Let $\widehat{X}_N=\frac{1}{N}\sum_{i=1}^NX_i$. Then 
    \begin{align*}
        P(|\widehat{X}_N-\mu|> \alpha\mu+\varepsilon)\leq 2\exp\{-\frac{N\alpha\varepsilon}{3R}\}.
    \end{align*}
\end{lemma}

\section{\texttt{greedy-bi} Algorithm}\label{apdx:greedy}

\begin{algorithm}[H]
    \caption{\texttt{greedy-bi}}\label{alg:greedy-bi}
    \begin{algorithmic}[1]
        \State \textbf{Input: }$\varepsilon$, $\tau$
        \State \textbf{Output: }$S\in U$
        \State $S\gets\emptyset$
        \While{$f(S) \leq (1 - \varepsilon) \tau$}
        \State $u \gets\arg\max_{x\in U} \Delta f(S,x)$
        \State $S \gets S \cup \{u\}$
        \EndWhile
        \textbf{return }$S$
    \end{algorithmic}
\end{algorithm}

\section{Additional Experiments}\label{apdx:add-exp}

\subsection{Implementation}

\textbf{Maximum Coverage.} 
In maximum coverage problems, the objective is to identify a set of fixed nodes that optimally maximize coverage within a network or graph. Given a graph $G = (V, E)$, where $V$ and $E$ respectively represent the set of vertices and nodes in the graph. Define a function $N : V \rightarrow 2^V$ as $N(v) = \{u, (u,v) \in E\}$, which represents the collection of neighbors of node $v$. Then the objective of this maximization problem can be defined as the monotone submodular function $f(S) = |\cup_{v \in S} N(v)|$. The dataset utilized in our maximum coverage experiments is the Twitch\_5000 of Twitch Gamers, which is a uniformly sampled subgraph of the Twitch Gamers dataset~\citep{twitch}, comprising 5,000 vertices (users) who speak English, German, French, Spanish, Russian, or Chinese. We aim to develop a solution with a high $f$ value exceeding a given threshold $\tau$ while ensuring a fair balance between users who speak different languages.

\paragraph{Set Covering}
For the datasets annotated with tags, the objective of set covering is to extract a diverse subset that maximizes the objective function $f(S) = |\cup_{x \in S} t(x)|$, where the function $t$ maps an element $x$ in a set $N$ to its corresponding tags $t(x)$. The dataset employed in our set covering experiments is a subset of Corel5k~\cite{corel}. For each item in the dataset, we randomly added a category from $\{0,2,3,4,5\}$ with a probability of $0.5$. Any item not assigned a random category was assigned category $1$. By ensuring a balanced distribution of solutions across categories, we aim to extract a representative set with a high $f$ value that surpasses a given threshold $\tau$.

\textbf{Experimental setup for max coverage.}
We implement our proposed algorithm \conv leveraging two subroutines provided in Appendix~\ref{appdx:alg_for_FSM}: \greedyalg{} (Algorithm~\ref{alg:fairness-bi}) and \threalglong{} (Algorithm~\ref{alg:thres-fairness-bi}). We compare our approach to the greedy baseline, \texttt{greedy-bi}, provided in Appendix~\ref{apdx:greedy}. 

To ensure a fair comparison based on the quality of the solutions, we use different default values for the parameter $\varepsilon$ in each algorithm. This is because each algorithm has a varying approximation ratio. Specifically, we set $\varepsilon = 0.1, \alpha=0.2, u_c = 1.1 / C, l_c = 0.9 / C$ for \texttt{greedy-bi} and \greedyalg{} (where $C$ is the number of groups). For $\threalglong$, we use $\varepsilon = 0.1$ while keeping the other parameters the same.

All the experiments are conducted on a single machine equipped with a 13th Gen Intel(R) Core(TM) i7-13700 CPU, 32GB of RAM, and Ubuntu 22.04.3 LTS. Each experiment with one set of parameters can be done in 120 seconds.

\subsection{Experiments Setup for set covering. } 

We implement our proposed algorithm \conv leveraging two subroutines provided in Appendix~\ref{appdx:alg_for_FSM}: \greedyalg{} (Algorithm~\ref{alg:fairness-bi}) and \threalglong{} (Algorithm~\ref{alg:thres-fairness-bi}). We compare our approach to the greedy baseline, \texttt{greedy-bi}, provided in Appendix~\ref{apdx:greedy}. 

To ensure a fair comparison based on the quality of the solutions, we use different default values for the parameter $\varepsilon$ in each algorithm. This is because each algorithm has a varying approximation ratio. Specifically, we set $\varepsilon = 0.1, \alpha=0.2, u_c = 1.1 / C, l_c = 0.9 / C$ for \texttt{greedy-bi} and \greedyalg{} (where $C$ is the number of groups). For $\threalglong$, we use $\varepsilon = 0.1$ while keeping the other parameters the same.

All the experiments are conducted on a single machine equipped with a 13th Gen Intel(R) Core(TM) i7-13700 CPU, 32GB of RAM, and Ubuntu 22.04.3 LTS. Each experiment with one set of parameters can be done in 30 seconds.

\begin{figure}[htbp!]
    \centering
    \begin{subfigure}{0.33\textwidth}
      \includegraphics[width=1\linewidth]{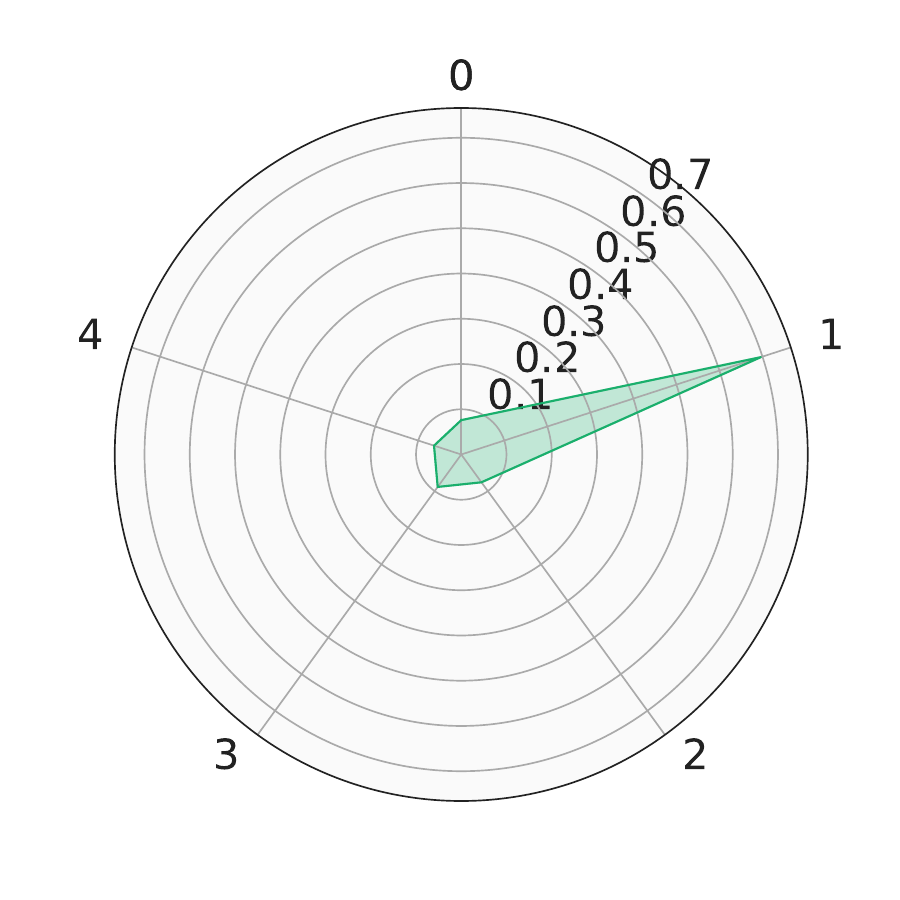}
      \caption{\texttt{greedy-bi}}
      \label{fig:radar-greedy-bi-set}
    \end{subfigure}%
    \begin{subfigure}{0.33\textwidth}
      \includegraphics[width=1\linewidth]{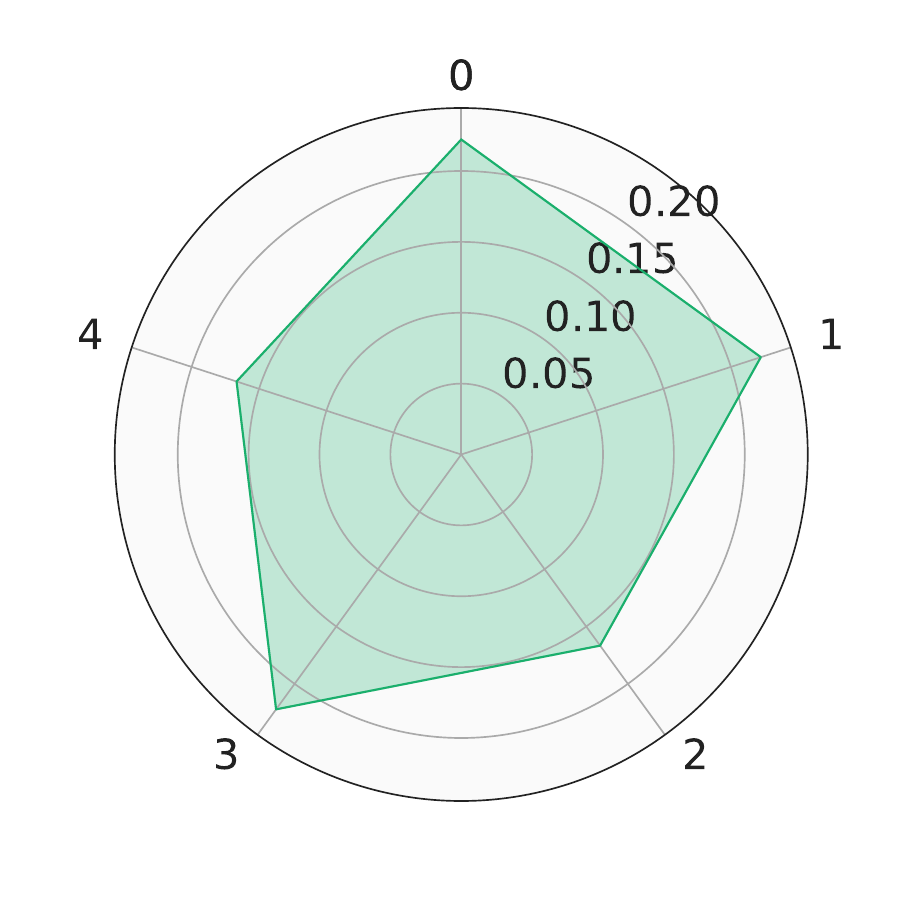}
      \caption{\texttt{greedy-fairness-bi}}
      \label{fig:radar-greedy-fairness-bi-set}
    \end{subfigure}
    \begin{subfigure}{0.33\textwidth}
      \includegraphics[width=1\linewidth]{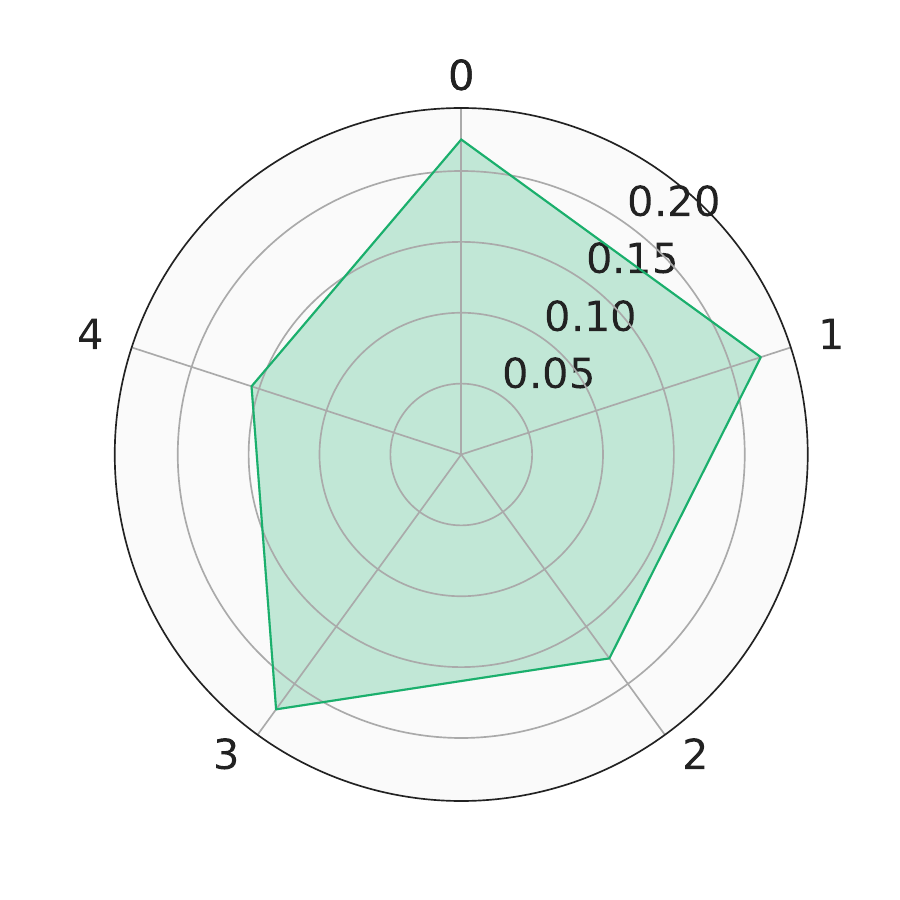}
      \caption{\texttt{threshold-fairness-bi}}
      \label{fig:radar-threshold-fairness-bi-set}
    \end{subfigure}

    \begin{subfigure}{0.33\textwidth}
      \includegraphics[width=1\linewidth]{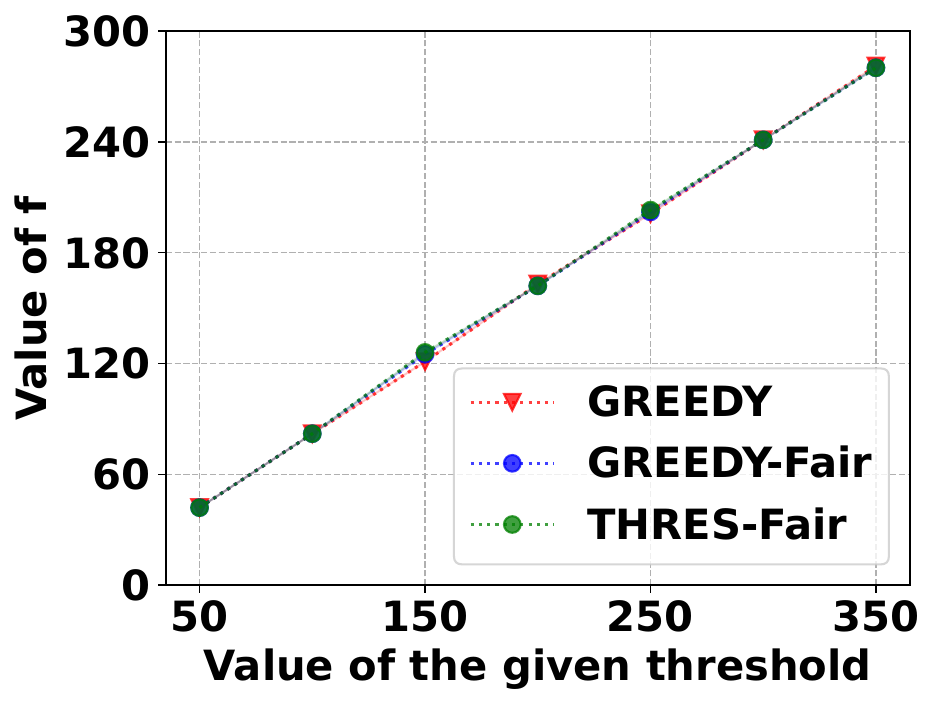}
      \caption{$f$}
      \label{fig:set-cover-tau-f}
    \end{subfigure}%
    \begin{subfigure}{0.33\textwidth}
      \includegraphics[width=1\linewidth]{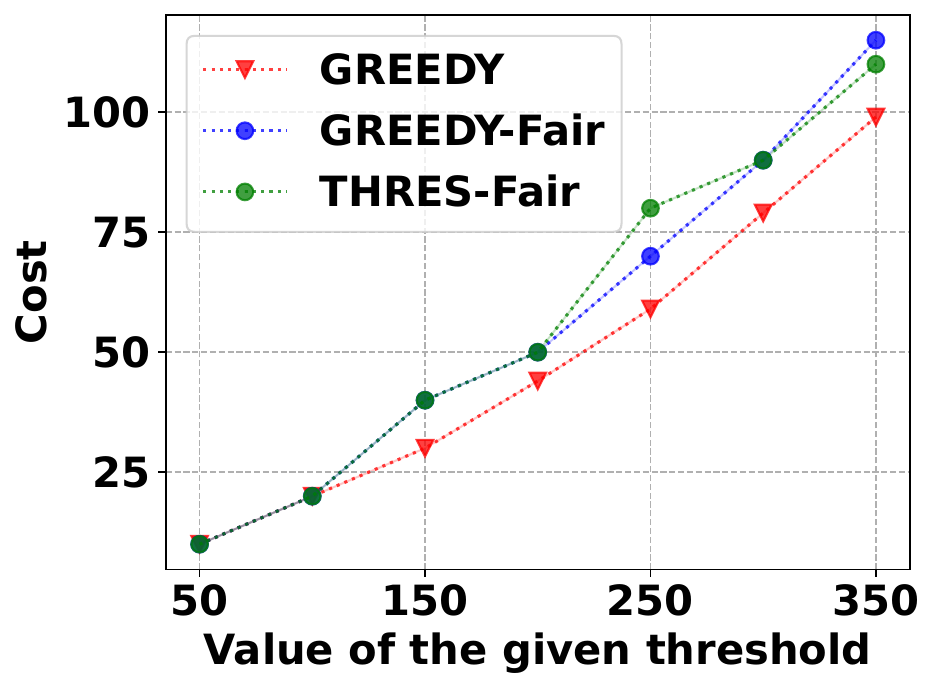}
      \caption{Cost}
      \label{fig:set-cover-tau-cost}
    \end{subfigure}
    \begin{subfigure}{0.33\textwidth}
      \includegraphics[width=1\linewidth]{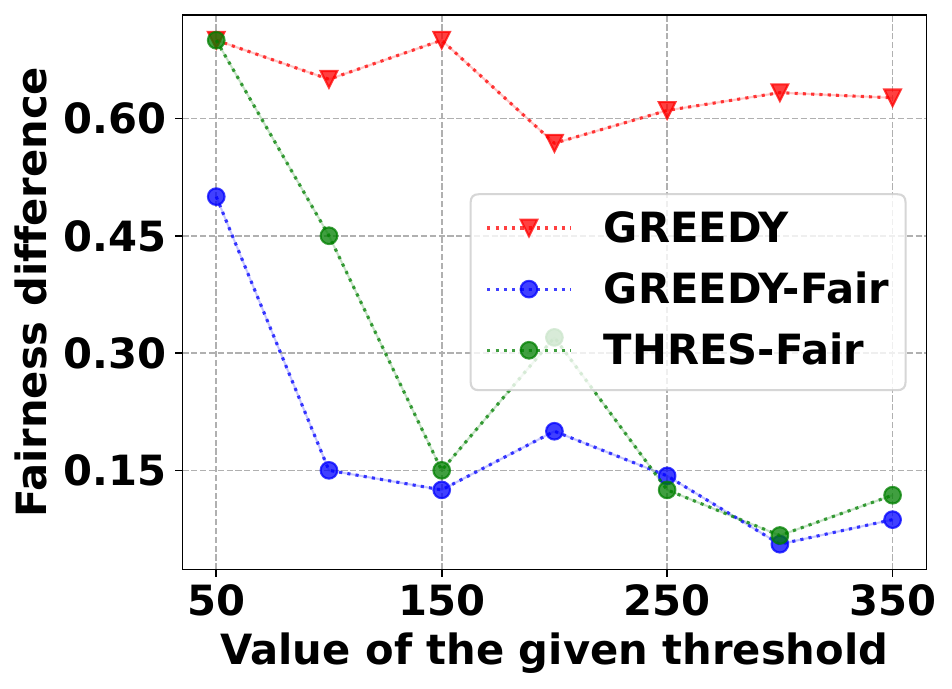}
      \caption{Fairness difference}
      \label{fig:set-cover-tau-diff}
    \end{subfigure}

    \caption{Performance comparison on the Corel dataset for Set Covering. \ref{fig:radar-greedy-bi-set}, \ref{fig:radar-greedy-fairness-bi-set}, \ref{fig:radar-threshold-fairness-bi-set} illustrate the distribution of images across various categories in the solutions produced by different algorithms with $\tau = 300$. $f$: the value of the objective submodular function. Cost: the size of the returned solution. Fairness difference: $(\max_c |S \cap U_c| - \min_c |S \cap U_c|) / |S|$}
    \label{fig:image-sum-central-exp}
\end{figure}

\subsection{Results}

Figures \ref{fig:radar-greedy-bi-set}, \ref{fig:radar-greedy-fairness-bi-set} and \ref{fig:radar-threshold-fairness-bi-set} showcase the distribution of images across various categories in the solutions produced by these algorithms with $\tau = 300$. Figures \ref{fig:set-cover-tau-f}, \ref{fig:set-cover-tau-cost} and \ref{fig:set-cover-tau-diff} present the performance of these algorithms ($f$ value, cost, and fairness difference) for varying values of $\tau$. As shown in Figure~\ref{fig:radar-greedy-bi-set}, with $\tau = 300$, over $70 \%$ of the pictures in the solution returned by \texttt{greedy-bi} are labeled as category `1'. While the solutions produced by \greedyalg{} and \threalglong{} exhibit way fairer distributions across various categories as shown in Figures \ref{fig:radar-greedy-bi-set} and \ref{fig:radar-greedy-fairness-bi-set}). Similarly, as the value of given $\tau$ increases, the magnitude of this difference also increases (see Figure \ref{fig:set-cover-tau-diff}). Figure \ref{fig:set-cover-tau-f} showcases that for all these algorithms the objective function value $f(S)$ scales almost linearly with the threshold $\tau$, which aligns with the theoretical guarantees of the approximation ratio.  

Notably, unlike the results presented in Section \ref{sec:exp}, our proposed algorithms achieve comparable costs to the \texttt{greedy-bi} solution (as shown in Figure~\ref{fig:set-cover-tau-cost}) on the Corel5k dataset. This is likely because the Corel5k dataset is less biased and the marginal gains for adding different elements are more uniform, compared to the Twitch Gamer dataset.

\section{Limitations}\label{apdx:limit}

As the first work focused on the fair submodular cover problems, we demonstrate the effectiveness of our proposed unified paradigm, \conv{}, by integrating it with four existing algorithms. While, a wider range of submodular maximization algorithms could potentially benefit from \conv{} under fairness constraints, leading to more empirically valuable methods for real-world applications. We consider this work a springboard for further exploration. Beyond the applications of social network analysis (maximum coverage) and image summarization (set covering) discussed in Section~\ref{sec:exp} and Appendix~\ref{apdx:add-exp}, our framework readily lends itself to scenarios like video summarization~\citep{gygli2015video}, movie recommendation~\citep{ohsaka2021approximation}, and others under fairness constraints.